\documentclass{article}
\usepackage[utf8]{inputenc} 
\usepackage[T1]{fontenc}    
\usepackage{url}            
\usepackage{booktabs}       
\usepackage{amsmath,amssymb,amsfonts}     
\usepackage{nicefrac}       
\usepackage{microtype}      
\usepackage{xcolor}         
\usepackage{latexsym}
\usepackage{subcaption}
\usepackage{graphicx}
\usepackage{wrapfig}
\usepackage{multirow}
\usepackage{breakcites}
\usepackage{fullpage}
\usepackage{wrapfig}
\usepackage{tikz}
\usetikzlibrary{decorations.pathreplacing}
\usepackage{setspace}
\usepackage{algorithm}
\usepackage[noend]{algpseudocode}
\usepackage[framemethod=tikz]{mdframed}
\usepackage{xspace}
\usepackage{pgfplots}
\usepackage{framed}
\usepackage{thm-restate}
\usepackage[backref=page]{hyperref}
\pgfplotsset{compat=1.5}

\newenvironment{proof}{\noindent{\bf Proof : \ }}{\hfill$\Box$\par\medskip}

\newtheorem{theorem}{Theorem}[section]
\newtheorem{corollary}[theorem]{Corollary}
\newtheorem{lemma}[theorem]{Lemma}

\newtheorem{definition}[theorem]{Definition}

\newtheorem{example}[theorem]{Example}

\newenvironment{proofof}[1]{\begin{trivlist} \item {\bf Proof
#1:~~}}
  {\qed\end{trivlist}}

\newcommand{\namedref}[2]{\hyperref[#2]{#1~\ref*{#2}}}
\newcommand{\thmlab}[1]{\label{thm:#1}}
\newcommand{\thmref}[1]{\namedref{Theorem}{thm:#1}}
\newcommand{\lemlab}[1]{\label{lem:#1}}
\newcommand{\lemref}[1]{\namedref{Lemma}{lem:#1}}

\newcommand{\corlab}[1]{\label{cor:#1}}

\newcommand{\seclab}[1]{\label{sec:#1}}
\newcommand{\secref}[1]{\namedref{Section}{sec:#1}}

\newcommand{\figlab}[1]{\label{fig:#1}}
\newcommand{\figref}[1]{\namedref{Figure}{fig:#1}}
\newcommand{\alglab}[1]{\label{alg:#1}}
\renewcommand{\algref}[1]{\namedref{Algorithm}{alg:#1}}

\newcommand{\exlab}[1]{\label{ex:#1}}
\newcommand{\exref}[1]{\namedref{Example}{ex:#1}}

\def \opt    {\mdef{\mathsf{opt}}}
\def \countsketch    {\mdef{\textsc{CountSketch}}}
\def \countsketchg    {\mdef{\textsc{GenCountSketch}}}
\def \estimator    {\mdef{\textsc{Estimator}}}

\def \sampler    {\mdef{\textsc{Sampler}}}

\def \sens    {\mdef{\textsc{Sens}}}
\def \hsens    {\mdef{\textsc{HSens}}}
\def \hestimator    {\mdef{\textsc{HEstimator}}}
\def \hsampler    {\mdef{\textsc{HSampler}}}

\def \oracle    {\mdef{\mathbb{O}}}

\def \a    {\mdef{\mathbf{a}}}
\def \A    {\mdef{\mathbf{A}}}
\def \B    {\mdef{\mathbf{B}}}

\def \H    {\mdef{\mathbf{H}}}

\def \M    {\mdef{\mathbf{M}}}
\def \calL    {\mdef{\mathcal{L}}}

\def \V    {\mdef{\mathbf{V}}}
\def \U    {\mdef{\mathbf{U}}}

\def \b    {\mdef{\mathbf{b}}}

\def \u    {\mdef{\mathbf{u}}}
\def \w    {\mdef{\mathbf{w}}}
\def \v    {\mdef{\mathbf{v}}}
\def \x    {\mdef{\mathbf{x}}}
\def \y    {\mdef{\mathbf{y}}}

\def \g    {\mdef{\mathbf{g}}}

\newcommand\norm[1]{\left\lVert#1\right\rVert}
\newcommand{\PPr}[1]{\ensuremath{\mathbf{Pr}\left[#1\right]}}

\newcommand{\Ex}[1]{\ensuremath{\mathbb{E}\left[#1\right]}}

\renewcommand{\O}[1]{\ensuremath{\mathcal{O}\left(#1\right)}}
\newcommand{\tO}[1]{\ensuremath{\tilde{\mathcal{O}}\left(#1\right)}}
\newcommand{\eps}{\varepsilon}
\DeclareMathOperator*{\var}{Var}
\DeclareMathOperator*{\nnz}{nnz}
\DeclareMathOperator*{\sgn}{sgn}
\DeclareMathOperator*{\tail}{tail}

\newcommand{\mdef}[1]{{\ensuremath{#1}}\xspace}  
\DeclareMathOperator*{\argmin}{argmin}

\DeclareMathOperator*{\polylog}{polylog}
\DeclareMathOperator*{\poly}{poly}

\newcommand{\ignore}[1]{}

\newif\ifnotes\notestrue 
\ifnotes
\newcommand{\samson}[1]{\textcolor{purple}{{\bf (Samson:} {#1}{\bf ) }} \marginpar{\tiny\bf
             \begin{minipage}[t]{0.5in}
               \raggedright S:
            \end{minipage}}}            							
\else
\newcommand{\samson}[1]{}
\fi

\hypersetup{
     colorlinks   = true,
     citecolor    = darkpastelgreen,
		 linkcolor		= teal
}

\makeatletter
\renewcommand*{\@fnsymbol}[1]{\textcolor{blue}{\ensuremath{\ifcase#1\or *\or \dagger\or \ddagger\or
 \mathsection\or \triangledown\or \mathparagraph\or \|\or **\or \dagger\dagger
   \or \ddagger\ddagger \else\@ctrerr\fi}}}
\makeatother

\providecommand{\email}[1]{\href{mailto:#1}{\nolinkurl{#1}\xspace}}

\definecolor{darkpastelgreen}{rgb}{0.01, 0.75, 0.24}

\usepackage{lmodern}

\title{Adaptive Sketches for Robust Regression with Importance Sampling}

\author{
Sepideh Mahabadi\thanks{Microsoft Research, Redmond. 
E-mail: \email{mahabadi@ttic.edu}}
\and
David P. Woodruff\thanks{School of Computer Science, Carnegie Mellon University.
E-mail: \email{dwoodruf@cs.cmu.edu}}
\and
Samson Zhou\thanks{School of Computer Science, Carnegie Mellon University.  
E-mail: \email{samsonzhou@gmail.com}}
}
\date{\today}

\begin{document}

\maketitle

\begin{abstract}
We introduce data structures for solving robust regression through stochastic gradient descent (SGD) by sampling gradients with probability proportional to their norm, i.e., importance sampling. Although SGD is widely used for large scale machine learning, it is well-known for possibly experiencing slow convergence rates due to the high variance from uniform sampling. On the other hand, importance sampling can significantly decrease the variance but is usually difficult to implement because computing the sampling probabilities requires additional passes over the data, in which case standard gradient descent (GD) could be used instead. In this paper, we introduce an algorithm that approximately samples $T$ gradients of dimension $d$ from nearly the optimal importance sampling distribution for a robust regression problem over $n$ rows. Thus our algorithm effectively runs $T$ steps of SGD with importance sampling while using sublinear space and just making a single pass over the data. Our techniques also extend to performing importance sampling for second-order optimization. 
\end{abstract}

\section{Introduction}
Given a matrix $\A\in\mathbb{R}^{n\times d}$ with rows $\a_1,\ldots,\a_n\in\mathbb{R}^d$ and a measurement/label vector $\b\in\mathbb{R}^n$, we consider the standard regression problem 
\[\underset{\x\in\mathbb{R}^d}{\min} \calL(\x):=\sum_{i=1}^n M(\langle\a_i,\x\rangle-b_i),\]
where $M:\mathbb{R}\to\mathbb{R}^{\ge 0}$ is a function, called a measure function, that satisfies $M(x)=M(-x)$ and is non-decreasing in $|x|$. 
An \emph{$M$-estimator} is a solution to this minimization problem and for appropriate choices of $M$, can combine the low variance of $L_2$ regression with the robustness of $L_1$ regression against outliers. 

The Huber norm, for example, is defined using the measure function $H(x)=\frac{x^2}{2\tau}$ for $|x|\le\tau$ and $H(x)=|x|-\frac{\tau}{2}$ for $|x|>\tau$, where $\tau$ is a threshold that governs the interpolation between $L_2$ loss for small $|x|$ and $L_1$ loss for large $|x|$. 
Indeed, it can often be more reasonable to have robust treatment of large residuals due to outliers or errors and Gaussian treatment of small residuals~\cite{Guitton99}. 
Thus the Huber norm is especially popular and ``recommended for almost all situations''~\cite{Zhang97}, because it is the ``most robust''~\cite{Huber92} due to ``the useful computational and statistical properties implied by the convexity and smoothness''~\cite{ClarksonW15} of its measure function, which is differentiable at all points. 

Since the measure function $H$ for the Huber norm and more generally, the measure function $M$ for many common measure functions is convex, we can consider the standard convex \emph{finite-sum form} optimization problem $\underset{\x\in\mathbb{R}^d}{\min} F(\x):=\frac{1}{n}\sum_{i=1}^n f_i(\x)$, where $f_1,\ldots,f_n:\mathbb{R}^d\to\mathbb{R}$ is a sequence of convex functions that commonly represent loss functions. 
Whereas gradient descent (GD) performs the update rule $\x_{t+1}=\x_t-\eta_t \nabla F(\x_t)$ on the iterate $\x_t$ at iterations $t=1,2,\ldots,T$, stochastic gradient descent (SGD)~\cite{RobbinsM51,NemirovskyY83,NemirovskiJLS09} picks $i_t\in[n]$ in iteration $t$ with probability $p_{i_t}$ and performs the update rule $\x_{t+1}=\x_t-\frac{\eta_t}{np_{i_t}}\nabla f_{i_t} (\x_t)$, where $\nabla f_{i_t}$ is the gradient (or a subgradient) of $f_{i_t}$ and $\eta_t$ is some predetermined learning rate. 
Effectively, training example $i_t$ is sampled with probability $p_{i_t}$ and the model parameters are updated using the selected example. 
The SGD update rule only requires the computation of a single gradient at each iteration and provides an unbiased estimator to the full gradient, compared to GD, which evaluates $n$ individual gradients in each iteration and is prohibitively expensive for large $n$. 
However, since SGD is often performed with uniform sampling, so that the probability $p_{i,t}$\footnote{In contrast to $p_{i,t}$, the term $p_{i_t}$ denotes the probability associated with the specific index $i_t$ chosen at time $t$.} of choosing index $i\in[n]$ at iteration $t$ is $p_{i,t}=\frac{1}{n}$ at all times, the variance introduced by the randomness of sampling a specific vector function can be a bottleneck for the convergence rate of this iterative process. 
Thus, the subject of variance reduction beyond uniform sampling has been well-studied in recent years~\cite{RouxSB12,JohnsonZ13,DefazioBL14,ReddiHSPS15,ZhaoZ15,DaneshmandLH16,NeedellSW16,StichRJ17,JohnsonG18,KatharopoulosF18,SalehiTC18,QianRGSLS19}. 

A common technique to reduce variance is importance sampling, where the probabilities $p_{i,t}$ are chosen so that vector functions with larger gradients are more likely to be sampled. 
One such setting of importance sampling is to set the probability of sampling a gradient with probability proportional to its $L_2$ norm, so that \[p_{i,t}=\frac{\norm{\nabla f_i(\x_t)}_2}{\sum_{j\in[t]}\norm{\nabla f_j(\x_t)}_2}.\]
Under these sampling probabilities, importance sampling gives variance 
\[\sigma_{opt,t}^2=\frac{1}{n^2}\left(\left(\sum_{i=1}^n\norm{\nabla f_i(\x_t)}_2\right)^2-n^2\cdot\norm{\nabla F(\x_t)}_2^2\right),\]
where we define the variance of a random vector $\v$ to be $\var(\v):=\Ex{\norm{\v}^2_2}-\norm{\Ex{\v}}_2^2$, and we define $\sigma_{opt,t}^2$ to be the variance of the random vector $\v$ produced at time $t$ by importance sampling. 

By comparison, the probabilities for uniform sampling $p_{i,t}=\frac{1}{n}$ imply $\sigma_t^2=\var\left(\frac{1}{np_{i_t,t}}\right)$ and thus the variance $\sigma_{uni,t}^2$ for uniform sampling satisfies
\[\sigma_{uni,t}^2=\frac{1}{n^2}\left(n\sum_{i=1}^n\norm{\nabla f_i(\x_t)}_2^2-n^2\cdot\norm{\nabla F(\x_t)}_2^2\right).\]
By the root mean square-arithmetic mean inequality, the variance of importance sampling is always at most the variance of uniform sampling, and can be significantly less.  
Hence $\sigma_{opt,t}^2\le\sigma_{uni,t}^2$, so that the variance at each step is reduced, possibly substantially, by performing importance sampling instead of uniform sampling. 

To see examples where uniform sampling an index performs significantly worse than importance sampling, consider $\nabla f_i(\x)=\langle\a_i,\x\rangle\cdot\a_i$. Then for $\A=\a_1\circ\ldots\circ\a_n$:
\begin{example}
\exlab{ex:mult:n}
When the non-zero entries of the input $\A$ are concentrated in a small number of vectors $\a_i$, uniform sampling will frequently sample gradients that are small and make little progress, whereas importance sampling will rarely do so. 
In an extreme case, the input $\A$ can contain exactly one non-zero vector $\a_i$ and importance sampling will always output the full gradient, whereas uniform sampling will only find the non-zero row with probability $\frac{1}{n}$, so that $\sigma_{uni,t}^2=n\cdot\sigma_{opt,t}^2$. 
\end{example}
\begin{example}
\exlab{ex:add:n}
It may be that all rows of $\A$ have large magnitude, but $\x$ is nearly orthogonal to most of the rows of $\A$, but is well-aligned with row $\a_r$.  
Then $\langle\a_i,\x\rangle\cdot\a_i$ is small in magnitude for most $i$, but $\langle\a_r,\x\rangle\cdot\a_r$ is large so uniform sampling will often output small gradients while importance sampling will output $\langle\a_r,\x\rangle\cdot\a_r$ with high probability, so that it can be that $\sigma_{uni,t}^2=\Omega(n)\cdot\sigma_{opt,t}^2$. 
\end{example}
\begin{example}
\exlab{ex:general}
More generally for a parameter $\nu\in[0,1]$, if a $\nu$-fraction of the $n$ gradients lengths are bounded by $\O{n}$ while the other $1-\nu$ fraction of the $n$ gradient lengths are bounded by $\poly(d)\ll n$, then the variance for uniform sampling satisfies $\sigma_{uni,t}^2=\O{\nu n^2}+\poly(d)$ while the variance for importance sampling satisfies $\O{\nu^2 n^2}+\poly(d)$. 
\end{example}

In fact, it follows from the Cauchy-Schwarz inequality that the importance sampling probability distribution is the \emph{optimal} distribution for variance reduction. 

However, computing the probability distribution for importance sampling requires computing the gradients in each round, which creates a ``chicken and egg'' problem because computing the gradients is too expensive in the first place, or else it is feasible to just run gradient descent. 
Unfortunately, computing the sampling probabilities in each iteration often requires additional passes over the data, e.g., to compute the gradients in each step, which is generally prohibitively expensive. 
This problem often prevents importance sampling from being widely deployed. 

In this paper, we overcome this problem by introducing efficient sketches for a wide range of $M$-estimators that can enable importance sampling \emph{without} additional passes over the data. 
Using our sketches for various measure functions, we give a time-efficient algorithm that \emph{provably} approximates the optimal importance sampling distribution within a constant factor.  
Thus we can surprisingly simulate $T$ steps of SGD with (nearly) the optimal sampling distribution, while only using a single pass over the data, which avoids the aforementioned problem.

\begin{restatable}{theorem}{thmsgdmain}
\thmlab{thm:sgd:main}
Given an input matrix $\A\in\mathbb{R}^{n\times d}$ whose rows arrive sequentially in a data stream along with the corresponding labels of a measurement vector $\b\in\mathbb{R}^d$, and a measure function $M$ whose derivative is a continuous union of piecewise constant or linear functions, there exists an algorithm that performs $T$ steps of SGD with variance within a constant factor of the optimal sampling distribution. 
The algorithm uses $\tO{nd^2+Td^2}$ pre-processing time and $Td^2\,\polylog(Tnd)$ words of space. 
\end{restatable}

For $T$ iterations, both GD and optimal importance sampling SGD require $T$ passes over the data, while our algorithm only requires a single pass over the data and uses sublinear space for $nd\gg Td^2$. 
We remark that although the number $T$ of iterations for SGD may be large, a major advantage of GD and SGD with importance sampling is a significantly smaller number of iterations than SGD with uniform sampling, e.g., as in \exref{ex:mult:n} and \exref{ex:add:n}, so we should expect $n\gg T$. 
In particular from known results about the convergence of SGD, e.g., see \thmref{thm:sgd:convergence}, if the diameter of the search space and the gradient lengths $\norm{\nabla f_i(x_t)}_2$ are both bounded by $\poly(d)$, then we should expect $T\propto\poly(d)$ even for uniform sampling. 
More generally, if $\nu n$ of the gradients have lengths $\Theta(n)$, while the remaining gradients have lengths $\poly(d)\ll n$, then \exref{ex:general} and \thmref{thm:sgd:convergence} show that the number of steps necessary for convergence for uniform sampling satisfies $T\propto \O{\nu^2 n^4}+\poly(d)$, while the number of steps necessary for convergence for importance sampling satisfies $T\propto \O{\nu^4 n^4}$. 
Thus for $\nu n=\O{n^C}$ for $C<1$, i.e., a sublinear number of gradients have lengths that exceed the input size, we have $T=\O{n^{4C}}$ and hence for $C<\frac{1}{4}$, we have roughly $T=o(n)$ steps are necessary for convergence for SGD with importance sampling. 

Finally, we show in the appendix that our techniques can also be generalized to perform importance sampling for second-order optimization. 
 
\subsection{Our Techniques}
In addition to our main conceptual contribution that optimal convergence rate of importance sampling for SGD can surprisingly be achieved (up to constant factors) without the ``chicken and egg'' problem of separately computing the sampling probabilities, we present a number of technical contributions that may be of independent interest. 
Our first observation is that if we were only running a single step of importance sampling for SGD, then we just want a subroutine that outputs a gradient $G(\langle\a_i,\x\rangle-b_i,\a_i)$ with probability proportional to its norm $\|G(\langle\a_i,\x\rangle-b_i,\a_i)\|_2$. 

\paragraph{$G$-sampler.}
In particular, we need an algorithm that reads a matrix $\A=\a_1\circ\ldots\circ\a_n\in\mathbb{R}^{n\times d}$ and a vector $\x\in\mathbb{R}^d$ given \emph{after processing} the matrix $\A$, and outputs (a rough approximation to) a gradient $G(\langle\a_i,\x\rangle-b_i,\a_i)$ with probability roughly 
\[\frac{\|G(\langle\a_i,\x\rangle-b_i,\a_i)\|_2}{\sum_{j=1}^n \|G(\langle\a_j,\x\rangle-b_j,\a_j)\|_2}.\]
We call such an algorithm a \emph{$G$-sampler} and introduce such a single-pass, memory-efficient sampler with the following guarantees:
\begin{restatable}{theorem}{thmsampler}
\thmlab{thm:sampler}
Given an $(\alpha,\eps)$-smooth gradient $G$, there exists an algorithm $\sampler$ that outputs a noisy vector $\v$ such that $\|\v-\a_i(\langle\a_i,x\rangle-b_i)\|_2\le\alpha\|\a_i(\langle\a_i,x\rangle-b_i)|\|_2$ and $\Ex{\v}=\a_i(\langle\a_i,x\rangle-b_i)$ is $\left(1\pm\O{\eps}\right)\frac{\|G(\langle\a_i,\x\rangle-b_i,\a_i)\|_2}{\sum_{j\in[n]}\|G(\langle\a_j,\x\rangle-b_j,\a_j)\|_2}+\frac{1}{\poly(n)}$. 
The algorithm uses $d^2\,\poly\left(\log(nT),\frac{1}{\alpha}\right)$ update time per arriving row and $Td^2\,\poly\left(\log(nT),\frac{1}{\alpha}\right)$ total bits of space. 
\end{restatable}
We say a gradient $G$ is $(\alpha,\eps)$-smooth if a vector $\u$ that satisfies $\|\u-\v\|_2\le\alpha\|\v\|_2$ implies that $(1-\eps)\|G(\v)\|_2\le\|G(\u)\|_2\le(1+\eps)\|G(\v)\|_2$. 
In particular, the measure functions discussed in \secref{sec:apps} have gradients that are $(\O{\eps},\eps)$-smooth. 
For example, the subgradient of the Huber estimator is $\a_i\cdot\sgn(\langle\a_i,\x\rangle-b_i)$ for $|\langle\a_i,\x\rangle-b_i|>\tau$, which may change sign when $\langle\a_i,\x\rangle-b_i$ is close to zero, but its norm will remain the same. 
Moreover, the form $G(\langle\a_i,\x\rangle-b_i,\a_i)$ necessitates that the gradient can be computed strictly from the two quantities $\langle\a_i,\x\rangle-b_i$ and $\a_i$. 
Thus \thmref{thm:sampler} implies that our algorithm can also compute a noisy vector $\v'$ such that $\|\v'-G(\langle\a_i,\x\rangle-b_i,\a_i)\|_2\le\eps\|G(\langle\a_i,\x\rangle-b_i,\a_i)\|_2$. 

Observe that an instance of $\sampler$ in \thmref{thm:sampler} can be used to simulate a single step of SGD with importance sampling and thus $T$ independent instances of $\sampler$ provide an oracle for $T$ steps of SGD with importance sampling. 
However, this na\"{i}ve implementation does not suffice for our purposes because the overall runtime would be $\tO{Tnd^2}$ so it would be more efficient to just run $T$ iterations of GD. 
Nevertheless, our $G$-sampler is a crucial subroutine towards our final algorithm and we briefly describe it here. 

An alternative definition of $G$-sampler is given in \cite{JayaramWZ22}. 
In their setting, the goal is to sample a coordinate $i\in[n]$ of a frequency vector $f$ with probability proportional to $G(f_i)$, where $G$ in their notation is a measure function rather than a gradient.  
However, because the $G$-sampler of \cite{JayaramWZ22} is not a linear sketch, their approach cannot be easily generalized to our setting where the sampling probability of each row $\a_i$ is a function of $\langle\a_i,\x\rangle$, but the vector $\x$ arrives after the stream is already processed. 

Furthermore, because the loss function $f$ may not be scale-invariant, then we should also not expect its gradient to be scale invariant at any location $\x\in\mathbb{R}^d$, i.e., $\nabla f(C\x)\neq C^p\cdot \nabla f(\x)$ for any constants $p,C>0$. 
Hence, our subroutine $\sampler$ cannot use the standard $L_p$ sampler framework used in \cite{JowhariST11,AndoniKO11,JayaramW18,MahabadiRWZ20}, which generally rescales each row of $\a_i$ by the inverse of a uniform or exponential random variable. 
A somewhat less common design for $L_p$ samplers is a level set and subsampling approach~\cite{MonemizadehW10,JiangLSWW21}, due to their suboptimal dependencies on the accuracy parameter $\eps$. 
Fortunately, because we require $\eps=\O{1}$ to achieve a constant factor approximation, we can use the level set and subsampling paradigm as a starting point for our algorithm. 
Because the algorithms of \cite{MonemizadehW10,JiangLSWW21} only sample entries of a vector implicitly defined from a data stream, our $G$-sampler construction must (1) sample rows of a matrix implicitly defined from a data stream and (2) permit updates to the sampling probabilities implicitly defined through multiplication of each row $\a_i$ with a vector $\x$ that only arrives after the stream is processed.

\paragraph{$G$-sampler through level sets and subsampling.}
To illustrate our method and simplify presentation here, we consider $L_2$ regression with gradient $\A_i\x:=\langle\a_i,\x\rangle\cdot\a_i$, by folding in the measurement vector $\b$ into a column of $\A$ -- our full algorithm in \secref{sec:g:sampler} handles both sampling distributions defined with respect to the norm of a general gradient $G$ in the form of \thmref{thm:sampler}, as well as an independent measurement vector $\b$. 

We first partition the rows of $\A$ into separate geometrically growing classes based on their $L_2$ norms, so that for instance, class $C_k$ contains the rows $\a_i$ of $\A$ such that $2^k\le\|\a_i\|_2<2^{k+1}$. 
We build a separate data structure for each class $C_k$, which resembles the framework for $L_p$ norm estimation~\cite{IndykW05}. 
We would like to use the approximate contributions of the level sets $\Gamma_1,\ldots,\Gamma_K$, with $K=\O{\frac{\log n}{\alpha}}$, toward the total mass $F_2(S)=\sum_{i=1}^n\norm{\A_i\x}_2$, where a level set $\Gamma_j$ is informally the set of rows $\a_i$ with $\norm{\A_i\x}_2\in\left[\frac{F_2(S)}{(1+\alpha)^{j-1}},\frac{F_2(S)}{(1+\alpha)^j}\right]$ and the contribution of a level set $\Gamma_j$ is $\sum_{i\in\Gamma_j}\norm{\A_i\x}_2$. 
Then we could first sample a level set $\Gamma_j$ from a class $C_k$ and then uniformly select a row $\a_i$ among those in $\Gamma_j$. 
Indeed, we can run a generalized version of the $L_2$ heavy-hitter algorithm $\countsketch$~\cite{CharikarCF04} on the stream $S$ to identify the level set $\Gamma_1$, since its rows will be heavy with respect to $F_2(S)$.  
However, the rows of the level sets $\Gamma_j$ for large $j$ may not be detected by $\countsketch$. 
Thus, we create $L=\O{\log n}$ substreams $S_1,\ldots,S_L$, so that substream $S_\ell$ samples each row of $\A$ with probability $2^{-\ell+1}$, and run an instance of $\countsketch$ on each substream $S_\ell$ to detect the rows of each level set and thus estimate the contribution of each level set.

\paragraph{Sampling from level sets with small contribution.}
However, there is still an issue -- some level sets have contribution that is too small to well-approximate with small variance. 
For example, if there is a single row with contribution $\frac{F_2(S)}{(1+\alpha)^j}$, then it might not survive the subsampling at a level $S_\ell$ that is used to detect it, in which case it will never be sampled. 
Alternatively, if it is sampled, it will be rescaled by a large amount, so that its level set will be sampled with abnormally large probability. 
Instead of handling this large variance, we instead add a number of dummy rows to each level set, to ensure that their contributions are all ``significant'' and thus be well-approximated. 

Now we have ``good'' approximations to the contributions of each level set within a class, so we can first select a level set with probability proportional to the approximate contributions of each level set and then uniformly sample a row from the level set. 
Of course, we may uniformly sample a dummy row, in which case we say the algorithm fails to acquire a sample. 
We show that the contribution added by the dummy rows is a constant fraction, so this only happens with a constant probability. 
Thus with $\O{\log\frac{1}{\delta}}$ constant number of independent samples, we can boost the probability of successfully acquiring a sample to $1-\delta$ for any $\delta\in(0,1]$. 
We then set $\delta=\frac{1}{\poly(n,T,d)}$. 
An illustration of the entire process can be viewed in \figref{fig:g:sampler}. 

\begin{figure*}[!htb]
\centering
\begin{tikzpicture}[scale=1]

\draw (-5.5,1.5) rectangle+(2,0.8);
\node at (-4.5,1.85){$G$-sampler};

\draw[dashed,purple,->] (-3.5,1.9) -- (-2,0.4);
\draw[dashed,purple,->] (-3.5,1.9) -- (-2,1.4);
\draw[->] (-3.5,1.9) -- (-2,2.4);
\draw[dashed,purple,->] (-3.5,1.9) -- (-2,3.4);

\draw (-2-0.2,-0.2) rectangle +(0.4+0.8,4.2);
\node at (-1.6,4.5){Classes};

\node at (-2+0.4,0.4){$\vdots$};
\draw[purple,dashed] (-2,1) rectangle+(0.8,0.8);
\node at (-2+0.4,1.4){$C_3$};
\draw (-2,2) rectangle+(0.8,0.8);
\node at (-2+0.4,2.4){$C_2$};
\draw[purple,dashed] (-2,3) rectangle+(0.8,0.8);
\node at (-2+0.4,3.4){$C_1$};

\draw[dashed,purple,->] (-1.2,2.4) -- (0,0.4);
\draw[dashed,purple,->] (-1.2,2.4) -- (0,1.4);
\draw[dashed,purple,->] (-1.2,2.4) -- (0,2.4);
\draw[->] (-1.2,2.4) -- (0,3.4);

\draw (0-0.2,-0.2) rectangle +(0.4+0.8,4.2);
\node at (0.4,4.5){Levels};

\draw[purple,dashed] (0,0) rectangle+(0.8,0.8);
\node at (0.4,0.4){$\Gamma_K$};
\node at (0.4,1.4){$\vdots$};
\draw[purple,dashed] (0,2) rectangle+(0.8,0.8);
\node at (0.4,2.4){$\Gamma_2$};
\draw (0,3) rectangle+(0.8,0.8);
\node at (0.4,3.4){$\Gamma_1$};

\draw[dashed,purple,->] (2-1.2,3.4) -- (2,0.4);
\draw[->] (2-1.2,3.4) -- (2,1.4);
\draw[dashed,purple,->] (2-1.2,3.4) -- (2,2.4);
\draw[dashed,purple,->] (2-1.2,3.4) -- (2,3.4);

\draw (2-0.2,-0.2) rectangle +(2.4+0.8,4.2);
\node at (3.2,4.5){Rows};

\draw[purple,dashed] (2,0) rectangle+(2.8,0.8);
\node at (2.8+0.6,0.4){Dummy rows};
\node at (2.8+0.6,1.4){$G(\langle\a_j,\x\rangle-b_j,\a_j)$};
\draw (2,1) rectangle+(2.8,0.8);
\node at (2.8+0.6,2.4){$\vdots$};
\draw[purple,dashed] (2,3) rectangle+(2.8,0.8);
\node at (2.8+0.6,3.4){$G(\langle\a_i,\x\rangle-b_i,\a_i)$};
\end{tikzpicture}
\caption{$G$-sampler first samples a specific class, then samples a specific level from the class, then samples a specific row from the level. Sampled items are denoted by solid lines while rejected items are denoted by dashed lines. The classes partition the rows $\a_k$ of the input $\A$ matrix by $\|\a_k\|_2$, while the levels further partition the rows $\a_i$ of each class based on $G(\langle\a_i,\x\rangle-b_i,\a_i)$.}
\figlab{fig:g:sampler}
\end{figure*}
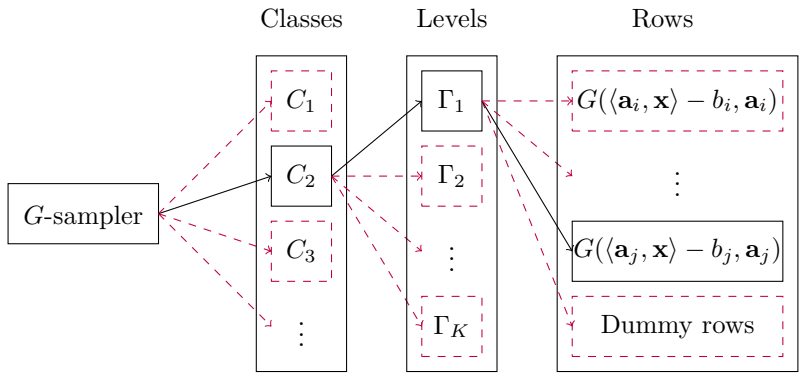

\paragraph{Unbiased samples.}
Unfortunately, $\countsketch$ using $\O{\frac{1}{\alpha^2}}$ buckets only guarantees additive $\alpha\,L_2(S)$ error to a particular row with constant probability. 
To achieve the standard ``for-all'' guarantee across all $n$ rows, an estimate for each row $\a_i$ is then output by taking the row with the median length across $\O{\log n}$ independent instances. 
However, the median row is no longer unbiased, which could potentially affect the convergence guarantees of SGD. 
Instead, we use $d$ separate instances of $\countsketch$, so that each instance handles a separate coordinate of the vector.  
Thus if the goal is to output a noisy estimate to $\a_i$, we have a separate $\countsketch$ report each coordinate $(\a_i)_j$, where $j\in[d]$. 
It can be shown that the median of each estimated coordinate is an unbiased estimate to the true value $(\a_i)_j$ of the coordinate because the probability mass function is symmetric about the true value for each coordinate. 
Moreover, the error to a single coordinate $(\a_i)_j$ may be large relative to the value of the coordinate in the case that $(\a_i)_j$ is not heavy with respect to $\{(\a_i)_j\}_{i\in[n]}$. 
However, we show that the ``overall'' error to all coordinates of $\a_i$ is small relative to $\|\a_i\|_2$, due to $\a_i$ being a ``heavy'' row at the appropriate subsampling level. 

%

\paragraph{Stochastic gradient descent with importance sampling.}
The main problem with the proposed $G$-sampler is that it requires reading the entire matrix $\A$ but it cannot be repeatedly used without incurring dependency issues. 
In particular, if a sampler at the first iteration of SGD outputs a gradient $\A_{i_1}\x_1$ that is used to construct $\x_2$, then $\x_2$ is not independent of the sampler and thus the same sampler should not be used to sample $\A_{i_2}\x_2$. 
This suggests that if we want to perform $T$ steps of SGD with importance sampling, then we would require $T$ separate data structures, which would require $Tnd$ time to construct for dense matrices, but then we might as well just perform full gradient descent!

Instead in \secref{sec:sgd:alg}, we partition the matrix $\A$ among multiple buckets and create a sampler for each bucket. 
Now as long as each bucket should have been sampled a single time, then we will have a fresh sampler with independent randomness for each time a new bucket is sampled. 
If we perform $T$ steps of SGD with importance sampling, then roughly $T$ buckets should suffice, but we cannot guarantee that each bucket is sampled a single time. 
For example, if only a single $\A_i$ is non-zero, then whichever bucket $\A_i$ is assigned to will be sampled every single time. 

Now the challenge is identifying the submatrices $\A_i=\a_i^\top\a_i$ that may be sampled multiple times, since we do not know the values of the vectors $\x_1,\ldots,\x_T$ a priori. 
Fortunately, we know that $\norm{\A_i\x_t}_2$ can only be large if $\a_i$ has high sensitivity, where we define the sensitivity for a row $\a_i$ in $\A$ to be the quantity $\max_{\x\in\mathbb{R}^d}\frac{\norm{\a_i^\top(\langle\a_i,\x\rangle)}_2}{\sum_{j=1}^n\norm{\a_j^\top(\langle\a_j,\x\rangle)}_2}$. 
Thus if a block is sampled multiple times, then one of its rows must have large sensitivity. 

Hence, we would like to identify the buckets that contain any row with sensitivity at least $\frac{1}{T}$ and create $T$ independent samplers for those buckets so that even if the same bucket is sampled all $T$ times, there will be a fresh sampler available. 
Crucially, the process of building separate buckets for the rows with the large sensitivities can be identified in just a single pass over the data. 

We remark that since each row has sensitivity $\max_{\x\in\mathbb{R}^d}\frac{\norm{\a_i^\top(\langle\a_i,\x\rangle)}_2}{\sum_{j=1}^n\norm{\a_j^\top(\langle\a_j,\x\rangle)}_2}$, then it can be shown that the sum of the sensitivities is $\O{d\log n}$ by partitioning the rows into $\O{\log n}$ classes $C_1,C_2,\ldots$ of exponentially increasing norm, so that $\a_i\in C_\ell$ if $2^{\ell}\le\norm{\a_i}_2<2^{\ell+1}$. 
We then note that the sensitivity of each row $\a_i\in C_\ell$ is upper bounded by $\max_{\x\in\mathbb{R}^d}\frac{|\langle\a_i,\x\rangle|}{\sum_{\a_j\in C_\ell}|\langle\a_j,\x\rangle|}$. 
However, this latter quantity is an $L_1$ sensitivity, whose sum is known to be bounded by $\O{d}$, e.g.,~\cite{ClarksonWW19}. 
Thus the sum of the sensitivities in each class is at most $\O{d}$ and so for a matrix $\A$ whose entries are polynomially bounded by $n$, the sum of the sensitivities is at most $\O{d\log n}$. 

Unfortunately, since the sensitivities sum to $\O{d\log n}$, there can be up to $Td$ rows with sensitivity at least $\frac{1}{T}$, so creating $T$ independent samplers corresponding to each of these rows would yield $\Omega(T^2d)$ samplers, which is a  prohibitive amount of space. 
Instead, we simply remove the rows with large sensitivities from the buckets and store them explicitly. 
We then show this approach still avoids any sampler from being used multiple times across the $T$ iterations while also enabling the data structure to just use $\tO{Td}$ samplers. 
Now since we can explicitly consider the rows with sensitivities roughly at least $\frac{1}{T}$, then we can use $\Theta(T)$ buckets in total to ensure that the remaining non-zero entries of $\A$ are partitioned evenly across buckets that will only require $\Theta(\log(Td))$ independent samplers. 
Intuition for our algorithm appears in \figref{fig:sgd}.

\begin{figure*}[!htb]
\centering
\begin{tikzpicture}[scale=1]

\draw (0,0) rectangle +(1.5,6);
\node at (1.5/2,6/2){$\A$};

\draw[->] (1.7,2) -- (2.3,2);

\draw (2.5,0) rectangle +(1.5,4); 
\node at (3.2,2.2){Light};
\node at (3.2,1.8){Rows};
\draw[red] (2.5,1) -- (4,1); 
\draw[red] (2.5,2) -- (4,2); 
\draw[red] (2.5,3) -- (4,3); 

\draw[->] (1.7,5.25) -- (2.3,5.25);

\draw (2.5,4.5) rectangle +(1.5,1.5); 
\node at (3.2,5.4){Heavy};
\node at (3.2,5.0){Rows};

\draw[->] (4.2,5.25) -- (5,5.25);

\draw (5.2,4.5) rectangle +(4,1.5); 
\node at (7.2,5.2){(Store rows explicitly)};

\draw[->] (4.2,2) -- (4.6,2);

\draw (4.8,-1) rectangle +(6.7,5.1); 
\draw [decorate, decoration = {brace}] (10,-0.2) --  (5,-0.2);
\node at (7.5,-0.7){$\O{\log(Td)}$};
\draw [decorate, decoration = {brace}] (10.3,4) --  (10.3,0);
\node at (10.95,2){$\O{T}$};

\draw (5,0.1) rectangle +(2,0.8); 
\draw (5,1.1) rectangle +(2,0.8); 
\draw (5,2.1) rectangle +(2,0.8); 
\draw (5,3.1) rectangle +(2,0.8); 

\node at (6,0.5){$G$-sampler};
\node at (6,1.5){$G$-sampler};
\node at (6,2.5){$G$-sampler};
\node at (6,3.5){$G$-sampler};

\node at (7.5,0.5){$\ldots$};
\node at (7.5,1.5){$\ldots$};
\node at (7.5,2.5){$\ldots$};
\node at (7.5,3.5){$\ldots$};

\draw (8,0.1) rectangle +(2,0.8); 
\draw (8,1.1) rectangle +(2,0.8); 
\draw (8,2.1) rectangle +(2,0.8); 
\draw (8,3.1) rectangle +(2,0.8); 

\node at (9,0.5){$G$-sampler};
\node at (9,1.5){$G$-sampler};
\node at (9,2.5){$G$-sampler};
\node at (9,3.5){$G$-sampler};

\node at (14,5.5){SGD steps};
\draw[->] (14,5) -- (14,-1);
\node at (14.5,4.8){$1$};
\node at (14.5,4.0){$2$};
\node at (14.5,3.2){$3$};
\node at (14.5,2){$\vdots$};
\node at (14.5,-0){$T-1$};
\node at (14.5,-0.8){$T$};

\draw[->] (13.5,4.8) -- (9.4,5.8);
\draw[->] (13.5,4.8) -- (11.7,3.9);

\draw[->] (13.5,4.0) -- (9.4,5.5);
\draw[->] (13.5,4.0) -- (11.7,3.6);

\draw[->] (13.5,3.2) -- (9.4,5.2);
\draw[->] (13.5,3.2) -- (11.7,3.3);

\node at (13.5,2){$\vdots$};
\end{tikzpicture}
\caption{Our SGD algorithm explicitly stores the heavy rows of the input matrix $\A$, i.e., the rows with sensitivity $\Omega\left(\frac{1}{T}\right)$. 
Our algorithm partitions the remaining rows of the input matrix into $\O{T}$ buckets and initiates $\O{\log(Td)}$ instances of $G$-samplers on each bucket.}
\figlab{fig:sgd}
\end{figure*}
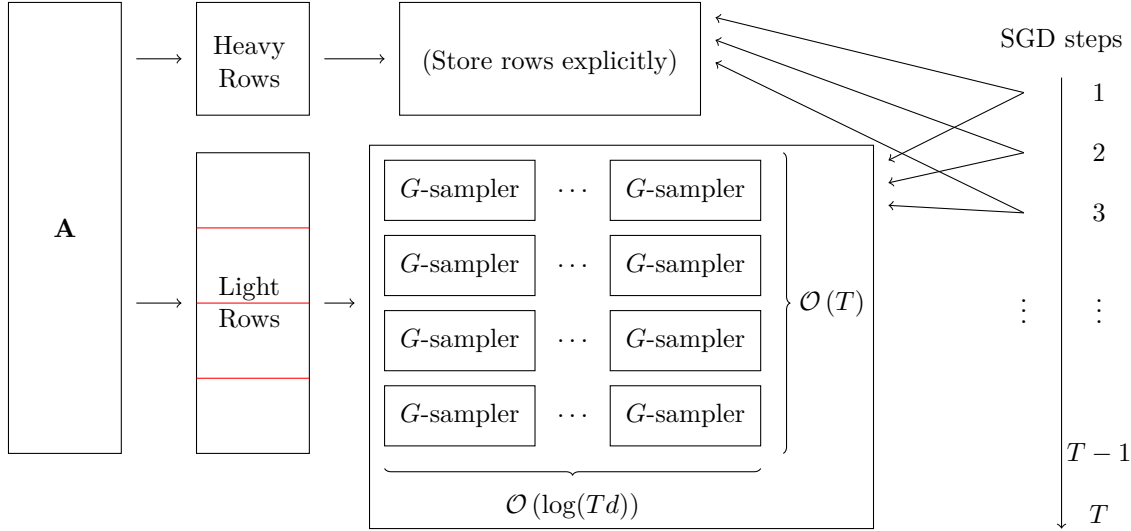

\subsection{Applications}
\seclab{sec:apps}
In this section, we discuss applications of our result to commonly used loss functions, such as $L_p$ loss or various $M$-estimators, e.g.,~\cite{ClarksonW15,ClarksonWW19,TukanWZBF22,PriceSZ22}. 

\paragraph{$L_1$ and $L_2$ regression.}
The $L_p$ regression loss function is defined using $f_i(\x)=|\a_i^\top\x-b_i|^p$. 
The case $p=2$ corresponds to the standard least squares regression problem, while $p=1$ corresponds to least absolute deviation regression, which is more robust to outliers than least squares, but also less stable and with possibly multiple solutions. 
For $p=1$, the subgradient is $\a_i\cdot\sgn(\langle\a_i,\x\rangle-b_i)$ while for $p=2$, the subgradient is $2\a_i(\langle\a_i,\x\rangle-b_i)$.  

\paragraph{Huber estimator.}
As previously discussed, Huber loss~\cite{Huber92} is commonly used, e.g.,~\cite{Zhang97,ClarksonW15}, to achieve Gaussian properties for small residuals~\cite{Guitton99} and robust properties for large residuals due to outliers or errors. 
The Huber estimator is also within a constant factor of other $M$-estimators that utilize the advantage of the $L_1$ loss function to minimize the impact of large errors/outliers and that of the $L_2$ loss function to be convex, such as the $L_1$-$L_2$ estimator and the Fair estimator~\cite{BosseAG16}. 
Given a threshold $\tau>0$, the Huber loss $H$ is defined by $H(x)=\frac{x^2}{2\tau}$ for $|x|\le\tau$ and $H(x)=|x|-\frac{\tau}{2}$ for $|x|>\tau$. 
Thus the subgradient for $H$ is $\frac{\a_i}{\tau}(\langle\a_i,\x\rangle-b_i)$ for $|\langle\a_i,\x\rangle-b_i|\le\tau$ and $\a_i\cdot\sgn(\langle\a_i,\x\rangle-b_i)$ for $|\langle\a_i,\x\rangle-b_i|>\tau$.

\paragraph{Ridge regression.}
It is often desirable for a solution $\x$ to be sparse. 
The natural approach to encourage sparse solutions is to add a regularization $\lambda\|\x\|_0$ term to the loss function, for some parameter $\lambda>0$. 
However, since $\|\x\|_0$ is not convex, ridge regression is often used as a convex relaxation that encourage sparse solutions. 
The ridge regression loss function satisfies $f_i(\x)=(\a_i^\top\x-b_i)^2+\lambda\,\|\x\|_2^2$ for each $i\in[n]$, so that $\lambda$ regularizes the penalty term associated with the squared magnitude of $\x$. 
Higher values of $\lambda$ push the optimal solution towards zero, which leads to lower variance, as a particular coordinate has a smaller effect on the prediction. 
The gradient for the ridge regression loss function satisfies $\nabla f_i(\x)=2\a_i(\langle\a_i,\x\rangle-b_i)+2\lambda\x$. 

\paragraph{Lasso.}
Another approach that encourages sparsity is using the $L_1$ regularization instead of the $L_2$ regularization. 
Least absolute shrinkage and selection operator (Lasso) regression uses the loss function $f_i(\x)=(\a_i^\top\x-b_i)^2+\lambda\,\|\x\|_1$. 
Whereas the penalty term associated for ridge regression will drive down the Euclidean norm of $\x$ for larger $\lambda$, solutions with large $L_1$ norm are still possible if the mass of $\x$ is spread across a large number of coordinates. 
By contrast, the penalty term associated for Lasso drives down the total magnitude of the coordinates of $\x$. 
Thus, in this sense, Lasso tends to drive coordinates to zero and encourages sparsity, which does not usually happen for ridge regression. 
The subgradient for the Lasso regression loss function satisfies $\nabla f_i(\x)=2\a_i(\langle\a_i,\x-b_i)+2\lambda\sgn(\x)$, where we abuse notation by using $\sgn(\x)$ to denote the coordinate-wise sign of the entries of $\x$.  

\paragraph{Group lasso.}
\cite{yuan2006model} proposed Group Lasso as a generalization to Lasso. 
Suppose the weights in $\x$ can be grouped into $m$ groups: $\x^{(1)},\ldots,\x^{(m)}$. 
We group the columns of $\A=\a_1\circ\ldots\circ\a_n$ so that $\A^{(i)}$ is the set of columns that corresponds to the weights in $\x^{(i)}$. 
The Group Lasso function is defined as $f_i(\x)=(\a_i^\top\x-b_i)^2+\lambda\sum_{j=1}^m\sqrt{G_j}\|\x^{(j)}\|_2$, where $G_j$ represents the number of weights in $\x^{(j)}$. 
Note that Group Lasso becomes Lasso for $m=n$. 

\subsection{Preliminaries}
For an integer $n>0$, we use $[n]$ to denote the set $\{1,2,\ldots,n\}$. 
We use boldfaced font for variables that represent either vectors of matrices and plain font to denote variables that represent scalars. 
We use the notation $\tO{\cdot}$ to suppress $\polylog$ factors, so that $f(T,n,d)=\tO{g(T,n,d)}$ implies that $f(T,n,d)\le g(T,n,d)\polylog(Tnd)$. 
Let $\A\in\mathbb{R}^{n\times d}$ and $\B\in\mathbb{R}^{m\times d}$. 
We use $\circ$ to denote vertical concatenation, so that $\A\circ\B=\begin{bmatrix}\A\\\B\end{bmatrix}$, and $\otimes$ to denote outer product, so that the $(i,j)$-th entry of the matrix $\u\otimes\v\in\mathbb{R}^{m\times n}$ for $\u\in\mathbb{R}^m$ and $\v\in\mathbb{R}^n$ is $u_iv_j$. 
For a vector $\v\in\mathbb{R}^{n}$, we let $\norm{\v}^p_p=\sum_{i=1}^n v_i^p$ and $\norm{\v}_{\infty}=\max_i|v_i|$. 
For a matrix $\A\in\mathbb{R}^{n\times d}$, we denote the Frobenius norm of $\A$ by $\norm{\A}_F=\sqrt{\sum_{i=1}^n\sum_{j=1}^d A_{i,j}^2}$. 
We also use $\norm{\A}_p=\left(\sum_{i=1}^n\sum_{j=1}^d|A_{i,j}|^p\right)^{\frac{1}{p}}$.  
For a function $f$, we use $\nabla f$ to denote its gradient. 

\begin{definition}
A function $f:\mathbb{R}^d\to\mathbb{R}$ is \emph{convex} if $f(\x)\ge f(\y)+\langle\nabla f(\y),\x-\y\rangle$ for all $\x,\y\in\mathbb{R}^d$.
\end{definition}

\begin{definition}
A continuously differentiable function $f:\mathbb{R}^d\to\mathbb{R}$ is \emph{$\mu$-smooth} if 
\[\norm{\nabla f(\x)-\nabla f(\y)}_2\le\mu\norm{\x-\y}_2,\]
for all $\x,\y\in\mathbb{R}^d$. 
Then it follows, e.g., by Lemma 3.4 in \cite{Bubeck15}, that for every $\x,\y\in\mathbb{R}^d$,
\[|f(\y)-f(\x)-\langle\nabla f(\x), \y-\x\rangle|\le\frac{\mu}{2}\norm{\y-\x}_2^2.\]
\end{definition}

Recall that SGD offers the following convergence guarantees for smooth convex functions:
\begin{theorem}
\cite{NemirovskiJLS09,Meka17}
\thmlab{thm:sgd:convergence}
Let $F$ be a $\mu$-smooth convex function and $\x_\opt=\argmin F(\x)$. 
Let $\sigma^2$ be an upper bound for the variance of the unbiased estimator across all iterations and $\overline{\x_k}=\frac{\x_1+\ldots+\x_k}{k}$. 
Let each step-size $\eta_t$ be $\eta\le\frac{1}{\mu}$. 
Then for SGD with initial position $\x_0$, and any value of $k$, 
\[\Ex{F(\overline{\x_k})-F(\x_\opt)}\le\frac{1}{2\eta k}\norm{\x_0-\x_\opt}_2^2+\frac{\eta\sigma^2}{2}.\]
This means that $k=\O{\frac{1}{\eps^2}\left(\sigma^2+\mu\norm{\x_0-\x_\opt}_2^2\right)^2}$ iterations suffice to obtain an $\eps$-approximate optimal value by setting $\eta=\frac{1}{\sqrt{k}}$.
\end{theorem}

\section{\texorpdfstring{$G$}{G}-Sampler Algorithm}
\seclab{sec:g:sampler}
In this section, we describe our $G$-sampler, which reads a matrix $\A=\a_1\circ\ldots\circ\a_n\in\mathbb{R}^{n\times d}$ and a vector $\x\in\mathbb{R}^d$ given after processing the matrix $\A$, and outputs a gradient $G(\langle\a_i,\x\rangle-b_i,\a_i)$ among the $n$ gradients $\{G(\langle\a_1,\x\rangle-b_1,\a_1),\ldots,G(\langle\a_n,\x\rangle-b_n,\a_n)\}$ with probability roughly $\frac{\|G(\langle\a_i,\x\rangle-b_i,\a_i)\|_2}{\sum_{j=1}^n \|G(\langle\a_j,\x\rangle-b_j,\a_j)\|_2}$. 
However, it is not possible to exactly return $G(\langle\a_i,\x\rangle-b_i,\a_i)$ using sublinear space; we instead return a vector $\v$ such that $\Ex{\v}=G(\langle\a_i,\x\rangle-b_i,\a_i)$ and $\|\v-G(\langle\a_i,\x\rangle-b_i,\a_i)\|\le\eps\|G(\langle\a_i,\x\rangle-b_i,\a_i)\|_2$. 
To achieve our $G$-sampler, we first require a generalization of the standard $L_2$-heavy hitters algorithm $\countsketch$~\cite{CharikarCF04}, which we describe in \secref{sec:hh}. 
We then describe our $G$-sampler in full in \secref{sec:sampler}. 

\subsection{Heavy-Hitters}
\seclab{sec:hh}
Before describing our generalization of $\countsketch$, we first require the following $F_G$ estimation algorithm that generalizes both well-known frequency moment estimation algorithm of~\cite{AlonMS99,ThorupZ04} and symmetric norm estimation algorithm of \cite{BlasiokBCKY17} by leveraging the linear sketches used in those data structures to support ``post-processing'' with multiplication by any vector $\x\in\mathbb{R}^d$. 
\begin{theorem}
\thmlab{thm:estimator}
\cite{BlasiokBCKY17}
Given a constant $\eps>0$ and an $(\alpha,\eps)$-smooth gradient $G$, there exists a one-pass streaming algorithm $\estimator$ that takes updates to entries of a matrix $\A\in\mathbb{R}^{n\times d}$, as well as vectors $\x\in\mathbb{R}^d$ and $\b\in\mathbb{R}^d$ that arrive after the stream, and outputs a quantity $\hat{F}$ such that $(1-\eps)\sum_{i\in[n]}\|G(\langle\a_i,\x\rangle-b_i,\a_i)\|_2\le\hat{F}\le(1+\eps)\sum_{i\in[n]}\|G(\langle\a_i,\x\rangle-b_i,\a_i)\|_2$. 
The algorithm uses $\frac{d^2}{\alpha^2}\,\polylog(nT)$ bits of space and succeeds with probability at least $1-\frac{1}{\poly(n,T)}$. 
\end{theorem}

We now describe a straightforward generalization of the $L_2$-heavy hitter algorithm $\countsketch$ so that (1) it can find the ``heavy rows'' of a matrix $\A=\a_1\circ\ldots\circ\a_n\in\mathbb{R}^{n\times d}$ rather than the ``heavy coordinates'' of a vector and (2) it supports post-processing multiplication by a vector $\x\in\mathbb{R}^d$ that arrives only after $\A$ is processed. 
Let $\A_i=\a_i\otimes\a_i\in\mathbb{R}^{d\times d}$ for all $i\in[n]$.  
We define $\tail(c)$ to be the $n-c$ rows that do not include the top $c$ values of $\norm{\A_i\x}_2$. 
For a given $\eps>0$, we say a block $\A_i$ with $i\in[n]$ is \emph{heavy} if $\norm{\A_i\x}_2\ge\eps\sum_{i\in\tail(2/\eps^2)} \norm{\A_i\x}_2$. 

The standard $\countsketch$ algorithm for finding the $L_2$-heavy hitters among the coordinates of a vector $\v$ of dimension $n$ works by hashing the universe $[n]$ across $\O{\frac{1}{\eps^2}}$ buckets. 
Each coordinate $i\in[n]$ is also given a random sign $\sigma_i$ and so the algorithm maintains the $\O{\frac{1}{\eps^2}}$ signed sums $\sum\sigma_ix_i$ across all the coordinates hashed to each bucket. 
Then to estimate $x_i$, the algorithm simply outputs $\sigma_iC_{h(i)}$, where $C_{h(i)}$ represents the counter corresponding to the bucket to which coordinate $i$ hashes. 
It can be shown that $\Ex{\sigma_iC_{h(i)}}=x_i$, where the expectation is taken over the random signs $\sigma$ and the choices of the hash functions. 
Similarly, the variance of the estimator can be bounded to show that with constant probability, the estimator has additive error $\O{\eps}\,\|\x\|_2^2$ to $x_i$ with constant probability. 
Thus if $x_i>\eps\,\|\x\|_2^2$, the algorithm will be able to identify coordinate $i$ as a heavy-hitter (in part by allowing some false positives). 
We give the algorithm in full in \algref{countsketch:matrix}. 

\begin{algorithm}[!htb]
\caption{Output heavy vectors $(\langle\a_i,\x\rangle)\a_i$, where $\x$ can be a vector that arrives after $\A$ is processed}
\alglab{countsketch:matrix}
\begin{algorithmic}[1]
\Require{Matrix $\A\in\mathbb{R}^{n\times d}$, vector $\x\in\mathbb{R}^d$, accuracy parameter $\eps>0$, failure parameter $\delta\in(0,1]$.}
\Ensure{Noisy vectors $\a_i^\top\a_i\x$ with $\norm{\a_i^\top\a_i\x}^2_2\ge\eps^2\sum_{i\in\tail(2/\eps^2)}\norm{\a_i^\top\a_i\x}^2_2$.}
\State{$b\gets\Omega\left(\frac{1}{\delta\eps^4}\right)$}
\State{Let $\mathcal{T}$ contain $b$ buckets, each initialized to the all zeros $\mathbb{R}^{d\times d}$ matrix.}
\State{Let $\sigma_i\in\{-1,+1\}$ be drawn from $4$-wise independent family for $i\in[n]$.}
\State{Let $h:[n]\to[b]$ be $2$-wise independent}
\State{\textbf{Process $\A$:}}
\State{Let $\A=\a_1\circ\ldots\circ\a_n$, where each $\a_i\in\mathbb{R}^{d}$.}
\For{each $j=1$ to $n$}
\State{$\A_j\gets\a_j\otimes\a_j$}
\State{Add $\sigma_j\A_j$ to the matrix in bucket $h(j)$.}
\EndFor
\State{Let $\M_j$ be the matrix in bucket $j$ of $\mathcal{T}$ for $i\in[r],j\in[b]$.}
\State{\textbf{Process $\x$:}}
\For{$j\in[b]$}
\State{$\v_j\gets\M_j\x$}
\EndFor
\State{On query $k\in[n]$, report $\sigma_k\v_{h(k)}$.}
\end{algorithmic}
\end{algorithm}

Thus \algref{countsketch:matrix} can be used to give the following guarantee by taking the median of the norms of $\O{\log(nT)}$ copies, as well as the vector that realizes the median. 
\begin{lemma}
\lemlab{lem:countsketch:heavy}
\cite{MahabadiRWZ20}
There exists an algorithm that uses $\O{\frac{d^2}{\eps^2}\,\log^2 n}$ space and outputs a set $S$ of indices so that with probability $1-\frac{1}{\poly(n,T)}$, for all $i\in[n]$, $i\in S$ if $\norm{\A_i\x}_2\ge\eps\sum_{j\in\tail(2/\eps^2)}\norm{\A_j\x}_2$ and $i\notin S$ if $\norm{\A_i\x}_2\le\frac{\eps}{2}\,\sum_{j\in\tail(2/\eps^2)}\norm{\A_j\x}_2$.  
The algorithm uses $\O{\frac{d^2}{\eps^2}\log^2(nT)}$ space. 
\end{lemma}

However, the vector that realizes the median of the norms may no longer be an unbiased estimate to each heavy-hitter. 
Unfortunately, we shall require unbiased estimates to each heavy-hitter, because we will use estimated heavy-hitters as unbiased gradients as part of SGD with importance sampling. 
Thus we give an additional algorithm so that for each $i\in S$ reported by \algref{countsketch:matrix}, the algorithm outputs an \emph{unbiased} estimate to the vector $(\langle\a_i,\x\rangle)\a_i$ with a ``small'' error, in terms of the total mass $\sum_{i\in\tail(2/\eps^2)}\norm{\A_i\x}_2$ excluding the largest $\frac{2}{\eps^2}$ rows. 

To that end, we instead run $d$ separate instances of $\countsketch$ to handle the $d$ separate coordinates of each heavy-hitter $\A_i\x$. 
We show that the median of each estimated coordinate is an unbiased estimate to the coordinate $(\A_i\x)_j$, since the probability mass function is symmetric about the true value for each coordinate. 
Furthermore, we show that although the error to a single coordinate $(\A_i\x)_j$ may be large compared to $|(\A_i\x)_j|$, the error is not large compared to $\sum_{i\in\tail(2/\eps^2)}\norm{\A_i\x}_2$.

\begin{restatable}{lemma}{lemcountsketchunbiased}
\lemlab{lem:countsketch:unbiased}
There exists an algorithm that uses $\O{\frac{d^2}{\eps^2}\log^2(nT)}$ space and outputs a vector $\y_i$ for each index $i\in[n]$ so that $|\norm{\y_i}_2-\norm{\A_i\x}_2|\le\eps\sum_{i\in\tail(2/\eps^2)}\norm{\A_i\x}_2$ and $\Ex{\y_i}=\A_i\x$ with probability at least $1-\frac{1}{\poly(n,T)}$. 
\end{restatable}
\begin{proof}
Given $\A\in\mathbb{R}^{n\times d}$, let $\A_i=\a_i^\top\a_i$ for all $i\in[n]$. 
For a fixed coordinate $k\in[d]$, we define a vector $\v^{(k)}\in\mathbb{R}^n$ so that for each $i\in[n]$, the $i$-th coordinate of $\v^{(k)}$ is the $k$-th coordinate of $\A_i\x\in\mathbb{R}^d$. 

Suppose we run a separate $\countsketch$ instance on $\v^{(k)}$. 
For a fixed index $i\in[n]$, let $h(i)$ be the bucket of $\mathcal{T}$ to which $v^{(k)}_i$ hashes. 
For each $j\in[n]$, let $I_j$ be the indicator variable for whether $v^{(k)}_j$ also hashes to bucket $h(i)$, so that $I_j=1$ if $h(i)=h(j)$ and $I_j=0$ if $h(i)\neq h(j)$. 
Similarly for each $j\in[n]$, let $s_j$ be a random sign assigned to $j$, so that the estimate for $v^{(k)}_i$ by a single row of $\countsketch$ is 
\[\sum_{j\in[n]}s_is_jI_jv^{(k)}_j=v^{(k)}_i+\sum_{j:h(j)=h(i)}r_jv^{(k)}_j,\]
where $r_j=s_is_j$ satisfies $r_j=1$ with probability $\frac{1}{2}$ and $r_j=-1$ with probability $\frac{1}{2}$. 
Thus if $y_i$ is the estimate for $v^{(k)}_i$, then for any real number $u$, we have that
\[\PPr{y_i=v^{(k)}_i+u}=\PPr{y_i=v^{(k)}_i-u},\]
so that the probability mass function of $y_i$ is symmetric about $v^{(k)}_i$. 
Thus given $\ell$ independent instances of $\countsketch$ with estimates $y_i^{(k,1)},\ldots,y_i^{(k,\ell)}$ for $v^{(k)}_i$ and any real numbers $u^{(1)},\ldots,u^{(\ell)}$,
\[\PPr{y_i^{(k,1)}=v^{(k)}_i+u^{(1)},\ldots,y_i^{(k,\ell)}=v^{(k)}_i+u^{(\ell)}}=\PPr{y_i^{(k,1)}=v^{(k)}_i-u^{(1)},\ldots,y_i^{(k,\ell)}=v^{(k)}_i-u^{(\ell)}}.\]
Therefore, the joint probability mass function is symmetric about $(v^{(k)}_i,\ldots,v^{(k)}_i)$ and so the median across the $\ell$ instances of $\countsketch$ is an unbiased estimator to $v^{(k)}_i$. 
Finally, we have due to the properties of $\countsketch$ that if each hash function $h$ maps to a universe of size $\O{\frac{1}{\eps^2}}$ and $\ell=\O{\log(nT)}$, then with probability at least $1-\frac{1}{\poly(T,n)}$, the output estimate for $v^{(k)}_i$ has additive error at most $\eps\cdot\left(\sum_{j\in\tail(2/\eps^2)}(v^{(k)}_i)^2\right)^{1/2}$. 

Thus using each of the estimated outputs across all $k\in[d]$, then for a fixed $i\in[n]$, we can output a vector $\y_i$ such that $\Ex{\y_i}=\A_i\x$ and with probability at least $1-\frac{1}{\poly(T,n)}$,  \[|\norm{\y_i}_2-\norm{\A_i\x}_2|\le\eps\cdot\left(\sum_{i\in\tail(2/\eps^2)}\norm{\A_i\x}_2^2\right)^{1/2}.\]
For a fixed $k\in[d]$, then our algorithm intends to hash the $k$-th coordinate of $\A_i\x\in\mathbb{R}^d$. 
However, since $\x$ is only given after the data structure is already formed and in particular, after $\A_i$ is given, then $\countsketch$ must hash the $k$-th row of $\A_i$ entirely, thus storing $\O{\frac{d}{\eps^2}\log^2(nT)}$ bits for each coordinate $k\in[d]$. 
Hence across all $k\in[d]$, the algorithm uses the total space $\O{\frac{d^2}{\eps^2}\log^2(nT)}$. 
\end{proof}

However if say, we want to identify the heavy gradients $(\langle\a_i,\x\rangle-b_i)\a_i$, then we create separate data structures for the constant (in $\x$) term $b_i\a_i$ and the linear term $(\a_i\otimes\a_i)\x$, using the same buckets, hash functions, and random signs.  
For the constant term data structure, we hash the scaled rows $b_i\a_i$ into $\O{\frac{1}{\eps^2}}$ buckets, so that each bucket contains a vector that represents the signed sum of the (scaled) rows of $\A$ that hash to the bucket. 
For the linear term data structure, we hash the outer products $\A_i:=\a_i\otimes\a_i$ into $\O{\frac{1}{\eps^2}}$ buckets, so that each bucket contains a vector that represents the signed sum of the matrices $\A_i$ that hash to the bucket. 
Once the vector $\x$ arrives after $\A$ is processed, then we can multiply each of the matrices stored by each bucket by $\x$. 
Since the signed sum is a linear sketch, this procedure is equivalent to originally taking the signed sums of the vectors $\A_i\x$. 
Similarly, by linearity, we can then take any linear combination of the two data structures to identify the heavy gradients $(\langle\a_i,\x\rangle-b_i)\a_i$. 

\subsection{\texorpdfstring{$G$}{G}-Sampler Algorithm}
\seclab{sec:sampler}
In this section, we first describe our $G$-sampler algorithm, where we sample a gradient $G(\langle\a_i,\x\rangle-b_i,\a_i)$ with probability proportional to $\|G(\langle\a_i,\x\rangle-b_i,\a_i)\|_2$. 
Given an accuracy parameter $\eps>0$, let $\alpha$ be a constant, parametrized by $\eps$, so that $(1-\eps)F_G(\v)\le F_G(\u)\le(1+\eps)F_G(\v)$, for any $\u$ with $\|\u-\v\|_2\le\alpha\|\v\|_2$. 
As our data structure will be a linear sketch, we focus on the case where we fold the measurement vector $\b$ into a column of $\A$, so that we want to output a gradient $\A_i\x:=(\a_i\otimes\a_i)\x$. 

Our algorithm first partitions the rows of $\A$ into classes, based on their $L_2$ norm. 
For example, if all entries of $\A$ are integers, then we define class $C_k:=\{\a_i\,:\,2^{k-1}\le\|\a_i\|_2<2^k\}$. 
We create a separate data structure for each class. 
We will use the $F_G$ estimation algorithm on each class to first sample a particular class. 
It then remains to sample a particular vector $(\a_i\otimes\a_i)\x$ from a class. 

Depending on the vector $\x$, the vectors $(\a_i\otimes\a_i)\x$ in a certain class $C_k$ can have drastically different $L_2$ norm. 
We define level set $\Gamma_j$ as the vectors that satisfy $(1+\eps)^{j-1}\le\|(\a_i\otimes\a_i)\x\|_2<(1+\eps)^j$. 
If we could estimate $|\Gamma_j|$, then we could estimate the contribution of each level set $\Gamma_j$ toward the overall mass $\sum_{i\in C_k}\|(\a_i\otimes\a_i)\x\|_2$, so we can then sample a specific level set $\Gamma_j$ from the class $C_k$. 
To that end, we create $L=\O{\log n}$ substreams, $S_1,\ldots,S_L$, so that we sample each row with probability $\frac{1}{2^{\ell-1}}$ in substream $S_\ell$. 

The point is that if the contribution of level set $\Gamma_j$ is ``significant'', then there exists a specific substream $S_\ell$ in which the vectors of $\Gamma_j$ will be likely detected by the heavy-hitter algorithm we introduced in \secref{sec:hh}, if they are sampled by $S_\ell$. 
We can then use these vectors that are output by the heavy-hitter algorithm to estimate the contribution of level set $\Gamma_j$. 
However, if the contribution of level set $\Gamma_j$ is not significant, then there may not be any vectors of $\Gamma_j$ that survive the sampling in substream $S_\ell$. 
Thus we add a number of ``dummy rows'' to each level set to insist that all level sets are significant, so that we can estimate their contributions. 

We then sample a level set $\Gamma_j$ with probability proportional to its contribution and uniformly select a (noisy) vector from the level set. 
If the selected vector is one of the original rows of the matrix, then we output the noisy vector. 
Otherwise, we say the sampler has failed. 
We show that the sampler only fails with constant probability, so it suffices to run $\O{\log\frac{1}{\delta}}$ independent instances to boost the probability of success to any arbitrary $1-\delta$. 
The algorithm for selecting a level set $\Gamma_j$ from a specific class $C_k$ appears in \algref{alg:sampler:uniform}. 

We first show that the dummy rows only contribute at constant multiple of the mass $F_G(S)=\sum_{i=1}^n\|\A_i\x\|_2$, where we assume for simplicity that all rows of $\A$ are in the same class.  
\begin{restatable}{lemma}{lemdummyadd}
\lemlab{lem:dummy:add}
Let $S$ be the input data stream with subsamples $S_1,\ldots,S_L$. 
Let $\widetilde{S}$ be the input data stream with the additional dummy rows and corresponding subsamples $\widetilde{S_1},\ldots,\widetilde{S_L}$. 
Then $2F_G(S)\ge F_G(\widetilde{S})\ge F_G(S)$. 
\end{restatable}
\begin{proof}
Since $\widetilde{S}$ includes all the rows of $S$, then $F_G(\widetilde{S})\ge F_G(S)$. 
Since each level $j\in[K]$ acquires $\O{\frac{(1+\alpha)^j\alpha^3}{\log n}}$ dummy rows that each contribute $\O{\frac{\widehat{M}}{(1+\alpha)^j\alpha^2}}$ to $F_G$ in $\widetilde{S}$, then each level of $F_G(S)$ contributes at most $\O{\frac{\widehat{M}\cdot\alpha}{\log n}}$ more to $F_G(\widetilde{S})$. 
Because $K=\O{\frac{\log n}{\alpha}}$, then the total additional contribution by the dummy rows is at most $\O{\widehat{M}}$. 
Since $\widehat{M}\le 2F_G(S)$, then it follows that for sufficiently small constant in the contribution of each dummy row, we have $F_G(\widetilde{S})-F_G(S)\le F_G(S)$ and thus, $F_G(\widetilde{S})\le 2F_G(S)$. 
\end{proof}

\begin{algorithm}[!htb]
\caption{$G$-sampler for a single class of rows}
\alglab{alg:sampler:uniform}
\begin{algorithmic}[1]
\Require{Rows $\a_1,\ldots,\a_n$ of a matrix $\A\in\mathbb{R}^{n\times d}$ with $2^k\le\|a_i\|_2<2^{k+1}$ for all $i\in[n]$, function $G$, accuracy parameter $\alpha$ for sampling parameter $\eps$}
\Ensure{Noisy row $\v$ with the correct sampling distribution induced by $G$}
\State{$\gamma$ uniformly at random from $[1/2,1]$, $K\gets\O{\frac{\log n}{\alpha}}$, $L\gets\O{\log n}$}
\For{$\ell\in[L]$}
\Comment{Processing stage}
\State{Form a stream $S_\ell$ by sampling each row with probability $2^{-\ell+1}$}
\State{Run $\countsketch^{(1)}_\ell$ with threshold $\O{\frac{\alpha^3}{\log n}}$ and failure probability $\frac{1}{\poly(n,T)}$ by creating a table $A^{(1)}_\ell$ with entries $\a_j^\top\a_j$ in $S_\ell$ and a table $B^{(1)}_\ell$ with entries $\a_j$}
\Comment{Identify heavy-hitters}
\State{Run $\countsketch^{(2)}_\ell$ with threshold $\O{\frac{\alpha^3}{\log n}}$ and failure probability $\frac{1}{\poly(n,T)}$ by creating a table $A^{(2)}_\ell$ with entries $\a_j^\top\a_j$ in $S_\ell$ and a table $B^{(2)}_\ell$ with entries $\a_j$ and separately considering coordinates after post-processing}
\Comment{Unbiased estimates of heavy-hitters, see \lemref{lem:countsketch:unbiased}}
\EndFor
\For{$\ell\in[L]$}
\Comment{Post-processing}
\State{Set $C^{(i)}_\ell=A^{(i)}_\ell\x+B^{(i)}_\ell$ with post-multiplication by $\x$ for $i\in\{1,2\}$}
\State{Query $\widehat{M}\in\left[M/2,2M\right]$, where $M=\sum_{i=1}^n\|G(\langle\a_i,\x\rangle-b_i,\a_i)\|_2$}
\For{$j\in[K]$}
\If{$j>\log_{(1+\alpha)}\frac{\log^2 n}{\alpha^3}$}
\State{Add $\O{\frac{(1+\alpha)^j\alpha^3}{\log n}}$ dummy rows that each contribute $\O{\frac{\widehat{M}}{(1+\alpha)^j\alpha^2}}$ to $F_G$}
\EndIf
\EndFor
\State{Let $H^{(i)}_\ell$ be the heavy rows of $C^{(i)}_\ell$ for $i\in\{1,2\}$ from $\countsketch^{(i)}_\ell$}
\EndFor
\For{$j\in[K]$}
\State{$L_j\gets\max\left(1,\log\frac{\alpha^2(1+\alpha)^j}{\log n}\right)$}
\State{Let $X_j$ be the estimated heavy-hitters $\v$ from $H^{(2)}_j$ that are reported by $H^{(1)}_j$ with $G(\v)$ in $\left[\frac{8\gamma\widehat{M}}{(1+\alpha)^{j+1}},\frac{8\gamma\widehat{M}}{(1+\alpha)^j}\right)$}
\If{$L_j=1$}
\State{$\widetilde{F_G}(\widetilde{S_j})\gets\sum_{\v\in X_j}\frac{8\gamma\widehat{M}}{(1+\alpha)^{j+1}}$}
\ElsIf{$L_j>1$ and $|X_j|>\frac{1}{\alpha^2}$}
\State{$\widetilde{F_G}(\widetilde{S_j})\gets\sum_{\v\in X_j}\frac{8\gamma\widehat{M}}{(1+\alpha)^{j+1}}\cdot 2^{L_j}$}
\Else
\State{$\widetilde{F_G}(\widetilde{S_j})\gets 0$}
\EndIf
\EndFor
\State{Sample $j\in[K]$ with probability $\frac{\widetilde{F_G}(\widetilde{S_j})}{\sum \widetilde{F_G}(\widetilde{S_j})}$}
\State{Sample $\v$ from $X_j$ with probability $\frac{1}{X_j}$}
\If{$\v$ is a dummy row}
\State{\Return $\bot$}
\Else
\State{\Return $\v$}
\EndIf
\end{algorithmic}
\end{algorithm}

We would now like to show that with high probability, each of the substreams have exponentially smaller mass $F_G(S_j)$. 
However, this may not be true. 
Consider a single row $\a_i$ that contributes a constant fraction of $F_G(S)$. 
Then even for $j=\log n$, the probability that $\a_i$ is sampled is roughly $\frac{1}{n}\gg\frac{1}{\poly(n)}$. 
Instead, we note that $\countsketch$ satisfies the stronger tail guarantee in \lemref{lem:countsketch:unbiased}.  
Hence for each $j\in[K]$, we define $S_j^{\tail(t)}$ to be the frequency vector $S_j$ with its $t$ largest entries set to zero and we show an exponentially decreasing upper bound on $F_G(\widetilde{S_j^{\tail(t)}})$. 
\begin{restatable}{lemma}{lemsubstreamupper}
\lemlab{lem:substream:upper}
With high probability, we have that for all $j\in[K]$, $F_G(\widetilde{S_j^{\tail(t)}})\le\frac{F_G(S)}{2^j}\,\log(nT)$ for $t=\O{\frac{\log n}{\alpha^3}}$.
\end{restatable}
\begin{proof}
Observe that the number of rows that exceed $\frac{\widehat{M}}{2^j}$ is at most $2^{j+1}$. 
Thus the expected number of rows that exceed $\frac{\widehat{M}}{2^j}$ sampled by $S_j$ is at most $\frac{1}{2}$. 
Hence by Chernoff bounds, the probability that the number of rows that exceed $\frac{\widehat{M}}{2^j}$ sampled by $S_j$ is more than $t=\O{\frac{\log n}{\alpha^3}}$ is $\frac{1}{\poly(nT)}$. 
\end{proof}

We also show that the estimated contribution of each level set (after incorporating the dummy rows) is a $(1+\alpha)$-approximation of the true contribution. 
\begin{restatable}{lemma}{lemsubstreammassapprox}
\lemlab{lem:substream:mass:approx}
With high probability, we have that for all $j\in[K]$, $(1-\eps)F_G(\widetilde{S_j})\le\widetilde{F_G}(\widetilde{S_j})\le(1+\eps)F_G(\widetilde{S_j})$.
\end{restatable}
\begin{proof}
Suppose that for each $j\in[K]$, level $j$ consists of $N_j$ rows and note that $N_j\ge\O{\frac{(1+\alpha)^j\alpha^3}{\log n}}$ elements due to the dummy rows. 
Each element is sampled with some probability $p_{L_j}$, where $L_j=\max\left(1,\log\frac{\alpha^2(1+\alpha)^j}{\log n}\right)$ and thus $p_{L_j}(1+\alpha)^j>1$ since $p_{L_j}=\frac{1}{2^{L_j}}$.  
Let $\widehat{N_j}$ be the number of items sampled in $\widetilde{S_{L_j}}$. 
We have $\Ex{2^{L_j}\cdot\widehat{N_j}}=N_j$ and the second moment is at most $N_j\cdot 2^{L_j}\le\frac{\alpha^2}{\log n}(N_j)^2$. 
Thus by Chernoff bounds with $\O{\log n}$-wise independence, we have that with high probability, 
\[(1-\O{\alpha})N_j\le 2^{L_j}\cdot\widehat{N_j}\le(1+\O{\alpha})N_j.\]
Each estimated row norm is a $(1+\alpha)$-approximation to the actual row norm due to \lemref{lem:substream:upper}. 
Thus by \lemref{lem:dummy:add}, we have that $F_G(\widetilde{S_j})\le 2F_G(S_j)$ so that each of the $\widehat{N_j}$ rows will be detected by the threshold of $\countsketch$ with the tail guarantee, i.e., \lemref{lem:countsketch:unbiased}.  
Moreover, we assume that a noisy row with $(1+\alpha)$-approximation to the row norm of the original vector suffices to obtain a $(1+\eps)$-approximation to the contribution of the row. 
Therefore, the result then follows in an ideal scenario where $G(\v)\in\left[\frac{\widehat{M}}{2^j},\frac{2\widehat{M}}{2^j}\right)$ if and only if the corresponding row $\a_i$ satisfies $G(\a_i)\in\left[\frac{\widehat{M}}{2^j},\frac{2\widehat{M}}{2^j}\right)$. 
Unfortunately, this may not be true because $G(\a_i)$ may lie near the boundary of the interval $\left[\frac{\widehat{M}}{2^j},\frac{2\widehat{M}}{2^j}\right)$ while the estimate $G(\v)$ has a value that does not lie within the interval. 
In this case, $G(\v)$ is used toward the estimation of some other level set. 

Hence, our algorithm randomizes the boundaries of the level sets $\left[\frac{4\gamma\widehat{M}}{2^j},\frac{8\gamma\widehat{M}}{2^j}\right)$ by choosing $\gamma\in[1/2,1)$ uniformly at random. 
Since the threshold of $\countsketch$ is $\O{\frac{\alpha^3}{\log n}}$ then the probability that each row $\a_i$ is misclassified over the choice of $\gamma$ is at most $\O{\eps}$. 
Moreover, if $\a_i$ is misclassified, then its contribution can only be classified into level set $j-1$ or $j+1$, inducing an incorrect multiplicative factor of at most two. 
Hence, the error due to the misclassification across all rows is at most $\O{\eps}$ fraction of $F_G(S_j)$ in expectation. 
By Markov's inequality, this error is a most $\eps$-fraction of $F_G(S_j)$ with probability at least $3/4$. 
Then by taking the median across $\O{\log(nT)}$ independent instances, we obtain high probability of success. 
\end{proof}

Finally, we show that each row is sampled with the correct distribution and is an unbiased estimate. 
\begin{restatable}{lemma}{lemsample}
\lemlab{lem:sample}
Suppose that $2^k<\|\a_i\|_2\le 2^{k+1}$ for all $i\in[n]$. 
Then the probability that \algref{alg:sampler:uniform} outputs a noisy vector $\v$ such that $\|\v-\a_i(\langle\a_i,x\rangle-b_i)\|_2\le\alpha\|\a_i(\langle\a_i,x\rangle-b_i)|\|_2$ with $\Ex{\v}=\a_i(\langle\a_i,x\rangle-b_i)$ is $p_\v=\left(1\pm\O{\eps}\right)\frac{G(\a_i(\langle\a_i,x\rangle-b_i))}{F_G(\widetilde{S})}+\frac{1}{\poly(nT)}$. 
\end{restatable}
\begin{proof}
Conditioned on the correctness of each of the estimates $\widetilde{F_G}(\widetilde{S_j})$, which occurs with high probability by \lemref{lem:substream:mass:approx}, the probability that the algorithm selects $j\in[K]$ is $\frac{\widetilde{F_G}(\widetilde{S_j})}{\sum_{j\in[K]}\widetilde{F_G}(\widetilde{S_j})}$. 
Conditioned on the algorithm selecting $j\in[K]$, then either the algorithm will choose a dummy row, or it will choose a row uniformly at random from the rows $\v\in X_j$, where $X_j$ is the set of heavy-hitters reported by $H_j$ with $L_2$ norm in $\left[\frac{8\gamma\widehat{M}}{(1+\alpha)^{j+1}},\frac{8\gamma\widehat{M}}{(1+\alpha)^j}\right)$. 
The latter event occurs with probability $\frac{\widetilde{F_G}(S_j)}{\widetilde{F_G}(\widetilde{S_j})}$. 
Due to the tail guarantee of $\countsketch$ in \lemref{lem:countsketch:unbiased}, we have that each heavy hitter $\v\in X_j$ corresponds to a vector $\a_i(\langle\a_i,x\rangle-b_i)$ such that $\|\v-\a_i(\langle\a_i,x\rangle-b_i)\|_2\le\eps\|\a_i(\langle\a_i,x\rangle-b_i)|\|_2$. 
Moreover, by \lemref{lem:countsketch:unbiased}, we have that $\Ex{\v}=\a_i(\langle\a_i,x\rangle-b_i)$. 
Hence the probability that vector $\v$ is selected is $\frac{(1\pm\O{\alpha})G(\a_i(\langle\a_i,x\rangle-b_i)}{\widetilde{F_G}(S_j)}$. 
\end{proof}
Putting things together, we have the guarantees of \thmref{thm:sampler} for our $G$-sampler.

\thmsampler*
\begin{proof}
We define a class $C_k$ of rows as the subset of rows of the input matrix $\A$ such that $2^k\le\|\a_i\|_2<2^{k+1}$. 
We first use the estimator algorithm in \thmref{thm:estimator} to sample a class $k$ of rows with probability $\frac{\sum_{\a_i\in C_k}G(\langle\a_i,x\rangle-b_i,\a_i)}{\sum_{j\in[n] }G(\langle\a_j,x\rangle-b_j,\a_j)}$. 
Once a class $C_k$ is selected, then outputting a row from $C_k$ under the correct distribution follows from \lemref{lem:sample}. 
The space complexity follows from storing a $d\times d$ matrix in each of the $\O{\frac{\log^2(nT)}{\alpha^3}}$ buckets in $\countsketch$ for threshold $\O{\frac{\alpha^3}{\log(nT)}}$ and high probability of success. 
\end{proof}

\section{SGD Algorithm and Analysis}
\seclab{sec:sgd:alg}
Before introducing our main SGD algorithm, we recall the following algorithm, that essentially outputs noisy version of the rows with high ``importance''. 
Although $\sampler$ outputs a (noisy) vector according to the desired probability distribution, we also require an algorithm that automatically does this for indices $i\in[n]$ that are likely to be sampled multiple times across the $T$ iterations. 
Equivalently, we require explicitly storing the rows with high sensitivities. 
\begin{theorem}
\cite{BravermanDMMUWZ20}
\thmlab{thm:sens}
Given a constant $\eps>0$, there exists an algorithm $\sens$ that returns all indices $i\in[n]$ such that $\sup_x\frac{|\a_i(\langle\a_i,\x\rangle-b_i)|}{\sum_{j=1}^n|\a_j(\langle\a_j,\x\rangle-b_j)|}\ge\frac{1}{200Td}$ for some $\x\in\mathbb{R}^n$, along with the vector $\a_i(\langle\a_i,\x\rangle-b_i)$. 
The algorithm requires a single pass over $\A=\a_1\circ\ldots\circ\a_n$, uses $\tO{nd^2+Td^2}$ runtime and $\tO{Td^2}$ space, and succeeds with probability $1-\frac{1}{\poly(n)}$. 
\end{theorem}
The quantity $\sup_x\frac{\norm{\a_i(\langle\a_i,\x\rangle-b_i)}_2}{\sum_{j=1}^n\norm{\a_j(\langle\a_j,\x\rangle-b_j)}_2}$ can be considered the \emph{sensitivity} of row $\a_i$ and can be interpreted as a measure of ``importance'' of the row $\a_i$ with respect to the other rows of $\A$.

We now proceed to describe our main SGD algorithm. 
For the finite-sum optimization problem $\underset{\x\in\mathbb{R}^d}{\min} F(\x):=\frac{1}{n}\sum_{i=1}^n G(\langle\a_i,\x\rangle-b_i,\a_i)$, where each $G$ is a piecewise function of a polynomial with degree at most $1$, recall that we could simply use an instance of $\sampler$ as an oracle for SGD with importance sampling. 
However, na\"{i}vely running $T$ SGD steps requires $T$ independent instances, which uses $Tnd$ runtime by \thmref{thm:sampler}. 
Thus, as our main theoretical contribution, we use a two level data structure by first implicitly partitioning the rows of matrix $\A=\a_1\circ\ldots\circ\a_n$ into $\beta:=\Theta(T)$ buckets $B_1,\ldots,B_{\beta}$ and creating an instance of $\estimator$ and $\sampler$ for each bucket. 
The idea is that for a given query $\x_t$ in SGD iteration $t\in[T]$, we first query $\x_t$ to each of the $\estimator$ data structures to estimate $\sum_{i\in B_j} G(\langle\a_i,\x\rangle-b_i,\a_i)$ for each $j\in[\beta]$. 
We then sample index $j\in[\beta]$ among the buckets $B_1,\ldots,B_{\beta}$ with probability roughly $\frac{\sum_{i\in B_j} G(\langle\a_i,\x\rangle-b_i,\a_i)}{\sum_{i=1}^n G(\langle\a_i,\x\rangle-b_i,\a_i)}$. 
Once we have sampled index $j$, it would seem that querying the instance $\sampler$ corresponding to $B_j$ simulates SGD, since $\sampler$ now performs importance sampling on the rows in $B_j$, which gives the correct overall probability distribution for each row $i\in[n]$.  
Moreover, $\sampler$ has runtime proportional to the sparsity of $B_j$, so the total runtime across the $\beta$ instances of $\sampler$ is $\tO{nd}$. 

However, an issue arises when the same bucket $B_j$ is sampled multiple times, as we only create a single instance of $\sampler$ for each bucket. 
We avoid this issue by explicitly accounting for the buckets that are likely to be sampled multiple times. 
Namely, we show that if
$\frac{G(\langle\a_i,\x_t\rangle-b_i,\a_i)}{\sum_{j=1}^n G(\langle\a_j,\x_t\rangle-b_j,\a_j)}<\O{\frac{1}{T}}$ for all $t\in[T]$ and $i\in[n]$, then by Bernstein's inequality, the probability that no bucket $B_j$ is sampled at least $2\log T$ times is at least $1-\frac{1}{\poly(T)}$. 
Thus we use $\sens$ to separate all such rows $\a_i$ whose sensitivities violate this property from their respective buckets and explicitly track the SGD steps in which these rows are sampled. 

The natural approach would be to create $T$ samplers for each of the rows with sensitivity at least $\Omega\left(\frac{1}{T}\right)$, ensuring that each of these samplers has access to fresh randomness in each of the $T$ SGD steps. 
However since the sensitivities sum to $\O{d\log n}$, there can be up to $\O{Td\log n}$ rows with sensitivity at least $\Omega\left(\frac{1}{T}\right)$, so creating $T$ samplers for each of these rows could create up to $\Theta(T^2d\log n)$ samplers, which is prohibitively expensive in $T$. 
Instead, we simply keep each row with sensitivity at least $\Omega\left(\frac{1}{T}\right)$ explicitly, while not including them in the bucket. 
Due to the monotonicity of sensitivities, the sensitivity of each row may only decrease as the stream progresses. 
In the case that a row had sensitivity at least $\Omega\left(\frac{1}{T}\right)$ at some point, but then no longer exceeds the threshold at some later point, then the row is given as input to the sampler corresponding to the bucket to which the row hashes and then the explicit storage of the row is deleted. 
This ensures we need only $\tO{Td}$ samplers while still avoiding any sampler from being used multiple times across the $T$ SGD steps. 
We give the algorithm in full in \algref{alg:sgd}. 

\begin{algorithm}[!htb]
\caption{Approximate SGD with Importance Sampling}
\alglab{alg:sgd}
\begin{algorithmic}[1]
\Require{Matrix $\A=\a_1\circ\ldots\circ\a_n\in\mathbb{R}^{n\times d}$, parameter $T$ for number of SGD steps.}
\Ensure{$T$ gradient directions.}
\State{\textbf{Preprocessing Stage:}}
\State{$\beta\gets\Theta(T)$ with a sufficiently large constant in the $\Theta$.}
\State{Let $h:[n]\to[\beta]$ be a uniformly random hash function.}
\State{Let $\B_j$ be the matrix formed by the rows $\a_i$ of $\A$ with $h(i)=j$, for each $j\in[\beta]$.}
\State{Create $\Theta(\log(Td))$ instances $\estimator_j$ and $\sampler_j$ for each $\B_j$ with $j\in[\beta]$ with $\eps=\frac{1}{2}$.}
\State{Run $\sens$ to find a set $L_0$ of rows with sensitivity at least $\Omega\left(\frac{1}{T}\right)$.}
\State{\textbf{Gradient Descent Stage:}}
\State{Randomly pick starting location $\x_0$}
\For{$t=1$ to $T$}
\State{Let $q_i$ be the output of $\estimator_j$ on query $\x_{t-1}$ for each $i\in[\beta]$.}
\State{Sample $j\in[\beta]$ with probability $p_j=\frac{q_j}{\sum_{i\in[\beta]} q_i}$.}
\If{there exists $i\in L_0$ with $h(i)=j$}
\State{Use $\estimator_j$, $L_0$, and $\sampler_j$ to sample gradient $\w_t=\widehat{\nabla f_{i_t}(\x_t)}$}
\Else
\State{Use fresh $\sampler_j$ to sample gradient $\w_t=\widehat{\nabla f_{i_t}(\x_t)}$}
\EndIf
\State{$\widehat{p_{i,t}}\gets\frac{\norm{\w_t}^2_2}{\sum_{j\in[\beta]} q_j}$}
\State{$\x_{t+1}\gets\x_t-\frac{\eta_t}{n\widehat{p_{i,t}}}\cdot\w_t$}
\EndFor
\end{algorithmic}
\end{algorithm}

The key property achieved by \algref{alg:sgd} in partitioning the rows and removing the rows that are likely to be sampled multiple times is that each of the $\sampler$ instances are queried at most once.

\begin{restatable}{lemma}{lemsgdfresh}
\lemlab{lem:sgd:fresh}
With probability at least $\frac{98}{100}$, each $t\in[T]$ uses a different instance of $\sampler_j$.
\end{restatable}
\begin{proof}
Let $C>0$ be a sufficiently large constant. 
For any $t\in[T]$ and $i\in[n]$, $G(\a_i(\langle\a_i,x\rangle-b_i))\ge\frac{1}{CT}\sum_{j\in[n]}G(\a_j(\langle\a_j,x\rangle-b_j))$ only if there exists a row in $\a_i\circ b_i$ whose sensitivity is at least $\frac{1}{CT}$. 
However, we have explicitly stored all rows $\a_i\circ b_i$ with sensitivity $\Omega\left(\frac{1}{T}\right)$ and removed them from each $G$-sampler. 

Thus, for all $j\in[\beta]$ so that $h(i)\neq j$ for any index $i\in[n]$ such that $G(\a_i(\langle\a_i,x\rangle-b_i))\le\frac{1}{CT}\sum_{k\in[n]}G(\a_k(\langle\a_k,x\rangle-b_k))$, we have 
\[\sum_{i: h(i)=j}G(\a_i(\langle\a_i,x\rangle-b_i))\le\frac{\log(Td)}{200T}\sum_{k\in[n]}G(\a_k(\langle\a_k,x\rangle-b_k)),\]
with probability at least $1-\frac{1}{\poly(Td)}$ by Bernstein's inequality and a union bound over $j\in[\beta]$, where $\beta=\Theta(T)$ is sufficiently large. 
Intuitively, by excluding the hash indices containing ``heavy'' matrices, the remaining hash indices contain only a small fraction of the mass with high probability.
 
We analyze the probability that any bucket containing rows with sensitivity less than $\O{\frac{1}{T}}$ are sampled more than $\Omega(T\log(Td))$ times, since we create $\O{T\log(Td)}$ separate $G$-samplers for each of these buckets. 
By a coupling argument and Chernoff bounds, the probability that any $j\in[\beta]$ with $\sum_{i: h(i)=j}G(\a_i(\langle\a_i,x\rangle-b_i))\le\frac{\log(Td)}{200T}\sum_{k\in[n]}G(\a_k(\langle\a_k,x\rangle-b_k))$ is sampled more than $200\log(Td)$ times is at most $\frac{1}{\poly(Td)}$ for any $t\in[T]$, provided there is no row with $h(i)=j$ whose sensitivity is at least $\frac{1}{CT}$. 
Thus, the probability that some bucket $j\in[\beta]$ is sampled more than $200\log(Td)$ times across $T$ steps is at most $\frac{1}{\poly(Td)}$.  

In summary, we would like to maintain $T$ separate instances of $G$-samplers for the heavy matrices and $\Theta(\log(Td))$ separate instances of $G$-samplers for each hash index that does not contain a heavy matrix, but this creates a $\Omega(T^2)$ space dependency. 
Instead, we explicitly store the heavy rows with sensitivity $\Omega\left(\frac{1}{T}\right)$, removing them from the heavy matrices, and manually perform the sampling, rather than rely on the $G$-sampler subroutine. 
There can be at most $\O{Td\log n}$ such rows, resulting in $\O{Td^2\log n}$ overall space for storing these rows explicitly. 
Since the resulting matrices are light by definition, we can maintain $\Theta(\log(Td))$ separate instances of $G$-samplers for each of the $\Theta(T)$ buckets, which results in $\tO{Td^2}$ space overall. 
With probability at least $\frac{98}{100}$, any hash index not containing a heavy matrix is sampled only once, so each time $t\in[T]$ has access to a fresh $G$-sampler. 
\end{proof}

\thmref{thm:sgd:main} then follows from \lemref{lem:sgd:fresh} and the sampling distribution guaranteed by each subroutine in \lemref{lem:sample}. 
In particular, \lemref{lem:sgd:fresh} crucially guarantees that each step $t\in[T]$ of SGD will receive a vector with fresh independent randomness. 
Moreover, we have that each (noisy) vector has small variance and is an unbiased estimate of a subgradient sampled from nearly the optimal importance sampling probability distribution. 

\thmsgdmain*
\begin{proof}
Consider \algref{alg:sgd}.
By \lemref{lem:sgd:fresh}, each time $t\in[T]$ uses a fresh instance of $\sampler_j$, so that independent randomness is used. 
A possible concern is that each instance $\estimator_j$ is not using fresh randomness, but we observe that the $\estimator$ procedures are only used in sampling a bucket $j\in[\beta]$; otherwise the sampling uses fresh randomness whereas the sampling is built into each instance of $\sampler_j$. 
By \thmref{thm:sampler}, each index $i$ is sampled with probability within a factor $2$ of the importance sampling probability distribution. 
By \thmref{thm:estimator}, we have that $\widehat{p_{i,t}}$ is within a factor $4$ of the probability $p_{i,t}$ induced by optimal importance sampling SGD. 
Note that $\w_t=\widehat{G(\langle\a_i,\x_t\rangle-b_i,\a_i)}$ is an unbiased estimator of $G(\langle\a_i,\x_t\rangle-b_i,\a_i)$ and $G(\w_t)$ is a $2$-approximation to $G(\x_t)$ by \thmref{thm:sampler}.  
Hence, the variance at each time $t\in[T]$ of \algref{alg:sgd} is within a constant factor of the variance $\sigma^2=\left(\sum_{i=1}^n G(\langle\a_i,\x_t\rangle-b_i,\a_i)\right)^2-\sum_{i=1}^n G(\langle\a_i,\x_t\rangle-b_i,\a_i)^2$ of optimal importance sampling SGD. 

By \thmref{thm:sampler}, \thmref{thm:estimator}, and \thmref{thm:sens}, the preprocessing time is $d^2\,\polylog(nT)$ for $\eps=\O{1}$ and $\beta=\Theta(T)$, but partitioning the non-zero entries of $\A$ across the $\beta$ buckets and the space used by the algorithm is $\tO{Td^2}$. 
Once the gradient descent stage of \algref{alg:sgd} begins, it takes $Td^2\,\polylog(n)$ time in each step $t\in[T]$ to query the $\beta=\Theta(T)$ instances of $\sampler$ and $\estimator$, for total time $Td^2\,\polylog(n)$. 
\end{proof}

\paragraph{Derandomization of the algorithm.}
To derandomize our algorithm, we first recall the following formulation of Nisan's pseudorandom generator.
\begin{theorem}[Nisan's PRG]
\cite{Nisan92}
\thmlab{thm:nisan}
Let $\mathcal{A}$ be an algorithm that uses $S=\Omega(\log n)$ space and $R$ random bits. 
Then there exists a pseudorandom generator for $\mathcal{A}$ that succeeds with high probability and runs in $\O{S\log R}$ bits. 
\end{theorem}
The goal of Nisan's PRG is to fool a small space tester by generating a number of pseudorandom bits in a read-once tape in place of a number of truly random bits. 
In the row-arrival model, the updates to each row $\a_i$ of $\A\in\mathbb{R}^{n\times d}$ arrive sequentially, so it suffices to use a read-once input tape. 
Thus a tester that is only allowed to $S$ space cannot distinguish between the output of our algorithm using true randomness and pseudorandom bits generated by Nisan's PRG. 
Since our algorithm uses $S=Td^2\,\polylog(Tnd)$ bits of space and $R=\poly(n,T,d)$ bits of randomness, then it can be randomized by Nisan's PRG while using $Td^2\,\polylog(Tnd)$ total space.

\section*{Acknowledgements}
David P. Woodruff and Samson Zhou were supported by a Simons Investigator Award and by the National Science Foundation under Grant No. CCF-1815840.

\bibliographystyle{alpha}
\bibliography{references}

\appendix
\section{Second-Order Optimization}
\seclab{sec:hessian}
In this section, we repurpose our data stucture that performs importance sampling for SGD to instead perform importance sampling for second-order optimization. 
Given a second-order optimization algorithm that requires a sampled Hessian $\H_t$, possibly along with additional inputs such as the current iterate $\x_t$ and the gradient $\g_t$ of $F$, we model the update rule by an oracle $\oracle(\H_t)$, suppressing other inputs to the oracle in the notation. 
For example, the oracle $\oracle$ corresponding to the canonical second-order algorithm Newton's method can be formulated as 
\[\x_{t+1}=\oracle(\x_t):=\x_t-[\H_t]^{-1}\g_t.\]
We can incorporate our sampling oracle into the update rule of any second-order optimization algorithm; thus we can focus our attention to the running time of sampling a Hessian with nearly the optimal probability distribution. 

Thus we use variations of our $G$-sampler for Hessians. 
In particular, we can store $\a_i^{\otimes 3}$ rather than $\a_i^\top\a_i=\a_i^{\otimes 2}$ for each entry of the $b$ buckets of $\countsketchg$ in \algref{countsketch:matrix}, where the notation $\v^{\otimes p}$ for a vector $\v\in\mathbb{R}^d$ represents the $p$-fold tensor of $\v$. 
We first require the following generalization of \thmref{thm:estimator}.
\begin{corollary}
\corlab{cor:hestimator}
\cite{BlasiokBCKY17}
Given a constant $\eps>0$ and a Hessian $H$ of a loss function that is the piecewise union of degree $p$ polynomials, there exists a one-pass streaming algorithm $\hestimator$ that takes updates to entries of a matrix $\A\in\mathbb{R}^{n\times d}$, as well as vectors $\x\in\mathbb{R}^d$ and $\b\in\mathbb{R}^d$ that arrive after the stream, and outputs a quantity $\hat{F}$ such that $(1-\eps)\sum_{i\in[n]}\|H(\langle\a_i,\x\rangle-b_i,\a_i))\|_2\le\hat{F}\le(1+\eps)\sum_{i\in[n]}\|H(\langle\a_i,\x\rangle-b_i,\a_i))\|_2$. 
The algorithm uses $\frac{d^p}{\alpha^2}\,\polylog(nT)$ bits of space and succeeds with probability at least $1-\frac{1}{\poly(n,T)}$. 
\end{corollary}
We remark that $\hestimator$ with parameter $p$ can be interpreted similarly as $\estimator$ with parameter $p$. 
Namely, $\hestimator$ maintains a signed sum of tensors $\a_i^{\otimes p}$, which can be reshaped into a matrix with dimension $\mathbb{R}^{d\times d^{p-1}}$. 

Then we have the following analog of \lemref{lem:countsketch:unbiased} to find the ``heavy'' matrices, corresponding to the ``important'' Hessians. 
\begin{lemma}
\lemlab{lem:countsketch:higher:unbiased}
Given an integer $p\ge 2$ and an integer $q\in(0,p]$, there exists an algorithm that uses $\O{\frac{d^p}{\eps^2}\log^2(nT)}$ space and outputs a vector $\y_i$ for each index $i\in[n]$ so that $|\norm{\y_i}_2-\norm{\a_i^{\otimes p}\x^{\otimes q}}_2|\le\eps\sum_{i\in\tail(2/\eps^2)}\norm{\a_i^{\otimes p}\x^{\otimes q}}_2$ and $\Ex{\y_i}=\a_i^{\otimes p}\x^{\otimes q}$ with probability at least $1-\frac{1}{\poly(n,T)}$. 
\end{lemma}

We now have the following analog to \thmref{thm:sampler}.
\begin{corollary}
\corlab{cor:hsampler}
Given a constant $\eps>0$ and a Hessian $H$ of a loss function that is the piecewise union of degree $p$ polynomials, there exists an algorithm $\hsampler$ that outputs a noisy matrix $\V$ such that $\|\V-H(\langle\a_i,\x\rangle-b_i,\a_i))\|_F\le\alpha\|H(\langle\a_i,\x\rangle-b_i,\a_i))\|_F$ and $\Ex{\V}=H(\langle\a_i,\x\rangle-b_i,\a_i))$ is $\left(1\pm\O{\eps}\right)\frac{\|H(\langle\a_i,\x\rangle-b_i,\a_i))\|_F}{\sum_{i\in[n]}\|H(\langle\a_i,\x\rangle-b_i,\a_i))\|_F}+\frac{1}{\poly(n)}$. 
The algorithm uses $d^p\,\poly\left(\log(nT),\frac{1}{\eps}\right)$ update time per arriving row and $Td^p\,\poly\left(\log(nT),\frac{1}{\eps}\right)$ total bits of space. 
\end{corollary}

Once the input is processed, the input $\x\in\mathbb{R}^d$ arrives and induces multiplication by a matrix with dimension $\mathbb{R}^{d^{p-1}\times d}$, which represents the tensor $\x^{\otimes p}$, to obtain a sampled matrix with dimension $\mathbb{R}^{d\times d}$ for second-order optimization. 
By comparison, $\estimator$ induces multiplication by a vector with dimension $\mathbb{R}^{d^{p-1}}$, which represents the tensor $\x^{\otimes(p-1)}$, to obtain a sampled vector with dimension $\mathbb{R}^{d}$ for first-order optimization. 
Similarly, the following two statements are simple corollaries of \thmref{thm:sampler} and \thmref{thm:sens}. 

\begin{corollary}
\cite{BravermanDMMUWZ20}
\corlab{cor:hleverage}
For a fixed $\eps>0$ and polynomial $f$ of degree $p$, there exists an algorithm $\hsens$ that returns all indices $i\in[n]$ such that $\frac{(1\pm\eps)\cdot\norm{f(\langle\a_i,\x-b_i\rangle)\cdot\a_i^\top\a_i}_F}{\sum_{j=1}^n\norm{f(\langle\a_j,\x-b_j\rangle)\cdot\a_i^\top\a_i}_F}\ge\frac{1}{200Td}$ for some $\x\in\mathbb{R}^n$, along with a matrix $\U_i:=f(\langle\a_i,\x\rangle)\cdot\a_i^\top\a_i+\V_i$, where $\norm{\V_i}_F\le\eps\cdot\norm{f(\langle\a_i,\x\rangle)\cdot\a_i^\top\a_i}_F$. 
The algorithm requires a single pass over $\A=\a_1\circ\ldots\circ\a_n$, uses $\tO{\nnz(\A)+T\poly(d)}$ runtime and $\tO{Td^p}$ space, and succeeds with probability $1-\frac{1}{\poly(n)}$. 
\end{corollary}
Thus the procedures $\hestimator$, $\hsampler$, and $\hsens$ generalize $\estimator$, $\sampler$, and $\sens$ in a straightforward manner by storing $\a_i^{\otimes p}$ where previously $\a_i^\top\a_i$ was stored for the SGD subroutines. 
Note that $\a_i^{\otimes p}$ can be reshaped into a $d\times d^{p-1}$ matrix so that subsequently, multiplication is done by $\x^{\otimes p}$, which can be reshaped as a $d^{p-1}\times d$ matrix.  

\begin{algorithm}[!htb]
\caption{Second-Order Optimization with Importance Sampling}
\alglab{hessian}
\begin{algorithmic}[1]
\Require{Matrix $\A=\a_1\circ\ldots\circ\a_n\in\mathbb{R}^{n\times d}$, parameter $T$ for number of sampled Hessians, oracle $\oracle$ that performs the update rule.}
\Ensure{$T$ approximate Hessians.}
\State{\textbf{Preprocessing Stage:}}
\State{$\beta\gets\Theta(T)$ with a sufficiently large constant in the $\Theta$.}
\State{Let $h:[n]\to[\beta]$ be a uniformly random hash function.}
\State{Let $\B_j$ be the matrix formed by the rows $\a_i$ of $\A$ with $h(i)=j$, for each $j\in[\beta]$.}
\State{Create an instance $\hestimator_j$ and $\hsampler_j$ for each $\B_j$ with $j\in[\beta]$ with $\eps=\frac{1}{2}$.}
\State{Run $\hsens$ to find a set $L_0$ of row indices and corresponding (noisy) outer products.}
\State{\textbf{Second-Order Optimization Stage:}}
\State{Randomly pick starting location $\x_0$}
\For{$t=1$ to $T$}
\State{Let $q_i$ be the output of $\hestimator_j$ on query $\x_{t-1}$ for each $i\in[\beta]$.}
\State{Sample $j\in[\beta]$ with probability $p_j=\frac{q_j}{\sum_{i\in[\beta]} q_i}$.}
\If{there exists $i\in L_0$ with $h(i)=j$}
\State{Use $\hestimator_j$, $\hsens$, and $\hsampler_j$ to sample Hessian $\H_t$.}
\Else
\State{Use $\hsampler_j$ to sample Hessian $\H_t=\widehat{\nabla f_{i_t}(\x_t)}$.}
\EndIf
\State{$\widehat{p_{i,t}}\gets\frac{\norm{\H_t}_F^2}{\sum_{j\in[\beta]} q_j}$}
\State{$\x_{t+1}\gets\oracle\left(\frac{1}{n\widehat{p_{i,t}}}\H_t\right)$}
\EndFor
\end{algorithmic}
\end{algorithm}

As before, observe that we could simply run an instance of $\hsampler$ to sample a Hessian through importance sampling, but sampling $T$ Hessians requires $T$ independent instances, significantly increasing the total runtime. 
We thus use the same two level data structure that partitions the rows of matrix $\A=\a_1\circ\ldots\circ\a_n$ into $\beta:=\Theta(T)$ buckets $B_1,\ldots,B_{\beta}$. 
We then create an instance of $\hestimator$ and $\hsampler$ for each bucket. 
For an iterate $\x_t$, we sample $j\in[\beta]$ among the buckets $B_1,\ldots,B_{\beta}$ with probability roughly $\frac{\sum_{i\in B_j}\norm{f(\langle\a_i,\x_t\rangle)\cdot\a_i^\top\a_i}_F}{\sum_{i=1}^n\norm{f(\langle\a_i,\x_t\rangle)\cdot\a_i^\top\a_i}_F}$ using $\hestimator$ and then querying $\hsampler_j$ at $\x_t$ to sample a Hessian among the indices partitioned into bucket $B_j$. 
As before, this argument fails when the same bucket $B_j$ is sampled multiple times, due to dependencies in randomness, but this issue can be avoided by using $\hsens$ to decrease the probability that each bucket is sampled. 
We give the algorithm in full in \algref{hessian}, which provides the following guarantees:

\begin{theorem}
\thmlab{thm:hessian:main}
Given an input matrix $\A\in\mathbb{R}^{n\times d}$ whose rows arrive sequentially in a data stream along with the corresponding labels of a measurement vector $\b\in\mathbb{R}^d$ and a loss function $M$ that is a polynomial of degree at most $p$, there exists an algorithm that outputs $T$ sequential Hessian samples with variance within a constant factor of the optimal sampling distribution. n. 
The algorithm uses $\tO{nd^p+Td^p}$ pre-processing time and $Td^p\,\polylog(Tnd)$ words of space. 
\end{theorem}

We remark that \algref{hessian} can be generalized to handle oracles $\oracle$ corresponding to second-order methods that require batches of subsampled Hessians in each iteration. 
For example, if we want to run $T$ iterations of a second-order method that requires $s$ subsampled Hessians in each batch, we can simply modify \algref{hessian} to sample $s$ Hessians in each iteration as input to $\oracle$ and thus $Ts$ Hessians in total. 
\end{document}